\documentclass{ecai} 

\usepackage{latexsym}
\usepackage{amssymb}
\usepackage{amsmath}
\usepackage{amsthm}
\usepackage{booktabs}
\usepackage{enumitem}
\usepackage{graphicx}
\usepackage{color}
\usepackage{algorithm}
\usepackage{amsfonts}       
\usepackage{nicefrac}       
\usepackage{microtype}      
\usepackage{xcolor}         
\usepackage[noend]{algpseudocode}
\usepackage{multirow}
\usepackage{caption}
\usepackage[inkscapeformat=png]{svg}
\usepackage{adjustbox}
\usepackage{tikz}

\theoremstyle{plain}

\newtheorem{theorem}{Theorem}
\newtheorem{proposition}{Proposition}
\newtheorem{lemma}{Lemma}

\newcommand{\Hquad}{\hspace{0.5em}}
\newcommand\numberthis{\addtocounter{equation}{1}\tag{\theequation}}
\newcommand{\BibTeX}{B\kern-.05em{\sc i\kern-.025em b}\kern-.08em\TeX}

\begin{document}

\captionsetup[figure]{
  aboveskip=10pt,  
  belowskip=10pt   
}
\captionsetup[table]{
  aboveskip=10pt,  
  belowskip=10pt   
}

\begin{frontmatter}


\paperid{1122} 


\title{Preserving the Privacy of Reward Functions in MDPs through Deception}


\author[A]{\fnms{Shashank}~\snm{Reddy Chirra}\thanks{Corresponding Author. Email: shashankc@smu.edu.sg.
}
\footnote{Work partially conducted when enrolled at the International Institute of Information Technology, Bangalore.\\
\textbf{Code:} https://github.com/shshnkreddy/DeceptiveRL}}

\author[A]{\fnms{Pradeep}~\snm{Varakantham}}
\author[B]{\fnms{Praveen}~\snm{Paruchuri}} 

\address[A]{Singapore Management University}
\address[B]{International Institute of Information Technology, Hyderabad}


\begin{abstract}
Preserving the privacy of preferences (or rewards) of a sequential decision-making agent when decisions are observable is crucial in many physical and cybersecurity domains. For instance, in wildlife monitoring, agents must allocate patrolling resources without revealing animal locations to poachers. This paper addresses privacy preservation in planning over a sequence of actions in MDPs, where the reward function represents the preference structure to be protected. Observers can use Inverse RL (IRL) to learn these preferences, making this a challenging task.

Current research on differential privacy in reward functions fails to ensure guarantee on the minimum expected reward and offers theoretical guarantees that are inadequate against IRL-based observers. To bridge this gap, we propose a novel approach rooted in the theory of deception. Deception includes two models: dissimulation (hiding the truth) and simulation (showing the wrong). Our first contribution theoretically demonstrates significant privacy leaks in existing dissimulation-based methods. Our second contribution is a novel RL-based planning algorithm that uses simulation to effectively address these privacy concerns while ensuring a guarantee on the expected reward. Experiments on multiple benchmark problems show that our approach outperforms previous methods in preserving reward function privacy. 
\end{abstract}

\end{frontmatter}

\setcounter{footnote}{2}

\section{Introduction}
\label{sec:intro}
In the realm of decision-making, particularly in situations involving resource allocation in the context of security, agents face the complex task of making choices that are potentially observable by external entities. These choices can carry substantial implications, revealing critical insights into the preferences (or significance) over different targets (or states in general); which can be strategically harnessed by observers in potentially harmful ways. The central challenge lies in optimizing these decisions while safeguarding the privacy of the agents' underlying preferences. Our specific focus is on addressing this challenge within the context of Reinforcement Learning (RL) based planners where the reward function represents the preferences that must be kept private. In such a case, it is crucial to recognize that a significant portion of the reward is embedded in the agents's decision-making policy. Therefore the agent must take actions that preserve the privacy of the reward while still achieving good performance. Take, for instance, green security games (GSGs) \cite{GSG_1}, where forest rangers patrol to monitor various animal populations. Poachers observing these patrols could exploit the information to locate and target animals. Thus, rangers must conduct effective surveillance while simultaneously deceiving the poachers. Similarly, in urban policing, cities are divided into regions, with each region assigned a reward based on factors such as crime rates, wealth, etc \cite{chen2013police, cornellpolice}. It is important for law enforcement to keep this reward function private for enhanced security.

\begin{figure}
\centering\includegraphics[width=2.8in,height=1.4in]{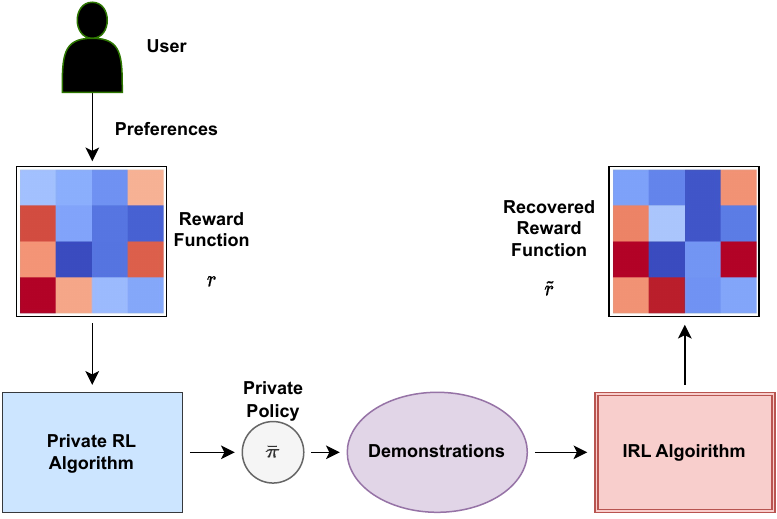}
    \caption{Flow of Information}
    \label{fig:if_flow}
\end{figure}

The potential of reverse-engineering the agent's reward forms the basis for the field of Inverse Reinforcement Learning (IRL)~\cite{ng_irl} which poses a substantial privacy risk. IRL has demonstrated the remarkable ability to reconstruct high-quality reward functions across various environments~\cite{gail, iql}. This concern underscores the importance of developing robust mechanisms to shield the agent's reward function, ensuring the integrity of decision-making and preventing potential privacy breaches. The problem of privacy preservation of the reward function is illustrated in Figure \ref{fig:if_flow}. First, the user defines the reward function $r$, which encodes their preferences. Next, a private RL algorithm learns a policy that maximizes the reward while simultaneously keeping the reward function private. An observer can then use an IRL algorithm to recover a reward function $\tilde{r}$ by observing demonstrations of the agent. If $\tilde{r}$ is of a high quality it will have properties very similar to $r$ making it feasible for an observer to estimate the preferences of the user. An ill-intentioned observer can use this information about the reward function to manipulate the agent to engage in undesired behaviours \cite{DBLP:conf/iclr/GleaveDWKLR20}. IRL covers the entirety of methods for recovering the reward function within our specific context, wherein the observer lacks additional information such as the nature of the reward function, domain knowledge, etc. However, a notable limitation of IRL lies in its assumption that demonstrations are not deceptive. To address this limitation, we introduce two modifications to IRL algorithms that account for deceptive demonstrations.

\begin{figure}
\centering\includegraphics[scale=0.4]{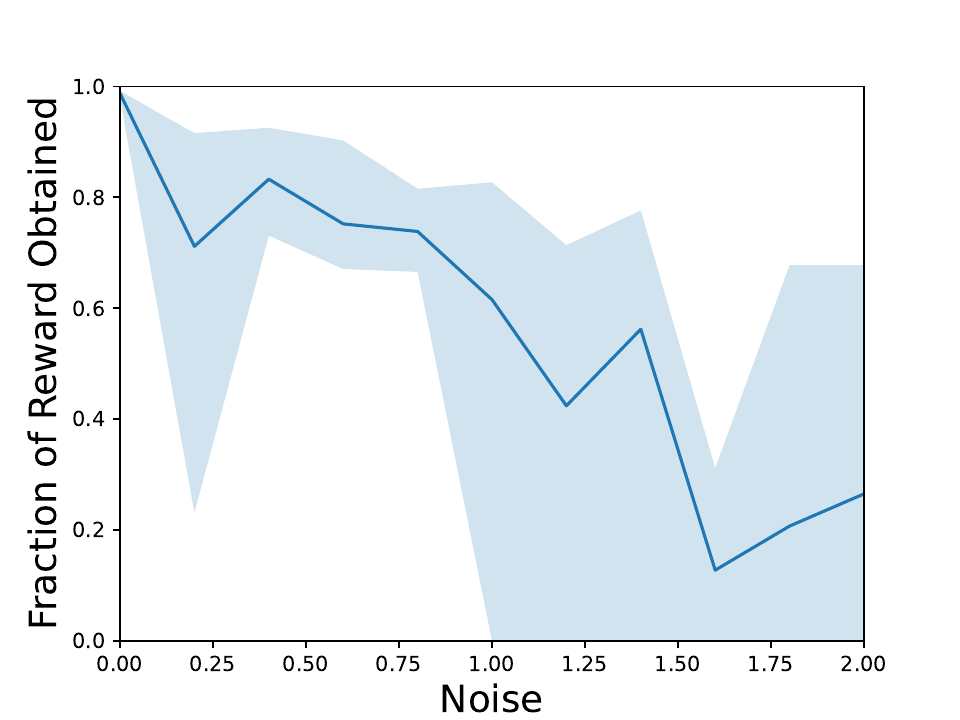}
    \caption{Expected Reward v/s Injected Noise of DQFN in the Four Rooms environment averaged over 5 seeds. The shaded region represents the max and min values. The large variance in the reward obtained underscores the difficulty in managing the privacy-reward tradeoff when using DQFN.}
    \label{fig:dp_rew}
\end{figure}

\noindent \textbf{\textit{Existing Work}}
tackles the problem in two ways:\\
\noindent (1) {\em Through the use of Differential Privacy (DP) methods}: DP-based methods \cite{Prakash_Husain_Paruchuri_Gujar_2022, vietri2020private, Zou} and the Deep Q-learning with Functional Noise (DQFN) algorithm \cite{dqfn} introduce noise to computations such as Q-functions, Value functions, and Policy Gradients. This addition of noise guarantees that reward functions within $l_{\infty}$-neighbourhood of each other return the same policy, making it difficult for an observer to reconstruct the exact reward function. In the context of reward reconstruction, these guarantees are ill-suited as (a) 
there are infinitely many reward functions that can explain the observed behaviour, (b) the $l_{\infty}$ and $l_{p}$ norms are not good metrics to use when comparing reward functions as two reward functions in the same $l_{\infty}$ neighbourhood may possess several other properties that pose a privacy leak such as ordering of polices (explained in Section \ref{sec:assesing}). This leads to a privacy leak in practice as highlighted in \cite{Prakash_Husain_Paruchuri_Gujar_2022}. This is in addition to the fact that these approaches lack built-in reward constraints make it difficult for a user to balance the tradeoff between expected reward and privacy without resorting to time-consuming hyper parameter searches. As shown in Figure~\ref{fig:dp_rew}, it is difficult to achieve a good privacy/reward trade-off due to high variance in expected reward as noise is increased. As highlighted in \cite{Prakash_Husain_Paruchuri_Gujar_2022}, another drawback of these methods is that the quality of the reward function recovered by an observer is independent of the noise added, undermining their effectiveness as a private algorithm.



 \noindent (2) {\em Through the use of deception}: Deception involves the act of intentionally creating or upholding false beliefs in the minds of others \cite{10.1093/acprof:oso/9780199577415.001.0001}. There are two primary approaches to deception: (a) "dissimulation" that relates to "hiding" the truth, and (b) "simulation" which entails providing false information to "mislead" the observer into believing something that is not true. The Max Entropy Intentional Randomization (MEIR) \cite{maxent} algorithm developed to preserve the privacy of the reward function is an existing "dissimulation" based deception algorithm (as shown later in this paper). Although the MEIR algorithm satisfies constraints on the expected rewards, we show that it leaks significant information about the reward function when faced with IRL-based observers. Other deception-based planning algorithms such as those discussed in \cite{10.5555/3061053.3061065, kulkarni2018resource, Masters2017DeceptiveP, topcu, deceptiverl, deceptiverl2}, are limited in their ability to tackle this problem. This limitation arises from their focus on optimizing deception for a singular trajectory, inadvertently disclosing information about the reward function across multiple trajectories. These methods also consider a different problem as discussed in Section \ref{app:related_work} in the Appendix. 

\noindent \textbf{\textit{Contributions:}}
Our contributions are as follows:
\begin{itemize}
\item \textbf{Theoretical Analysis of Privacy Leak for MEIR}: We demonstrate theoretically and intuitively the significant privacy vulnerabilities of the MEIR algorithm when faced with an IRL observer.
\item \textbf{Novel Max Misinformation Algorithm}: We introduce the innovative Max Misinformation (MM) Algorithm, designed to address the shortcomings of MEIR and DP-based methods. A key element is the introduction of an anti-reward function, enabling a balanced tradeoff between the expected value and the ability to deceive the observer.
\item \textbf{Effectiveness Against IRL algorithms}: We provide insights into why MM can robustly counteract observers utilizing various IRL algorithms, demonstrating its superiority in preserving privacy of the reward function. In addition, we experiment against two additional algorithms based on IRL that an observer might use if they know they are being deceived, and demonstrate the robustness of the proposed algorithm in this case as well.
\item \textbf{Comprehensive Evaluation}: To gauge the effectiveness of our algorithm, we rigorously evaluate it against IRL-based observers across diverse benchmark environments. We measure the quality of the recovered reward functions in comparison to the original reward using the Rollout method \cite{epic}, Pearson Correlation, and the Equivalent-Policy Invariant Comparison (EPIC) distance \cite{epic}. Our conclusive findings highlight that the MM algorithm outperforms existing deception-based and DP-based algorithms, firmly establishing its efficacy in maintaining the privacy of the reward.
\end{itemize}

\section{Background}
\label{sec:background}


We provide a brief overview of the relevant decision-making models (Markov Decision Process (MDP) and Deceptive RL), Max Casual Entropy Inverse Reinforcement Learning (MCE-IRL) and Max Entropy Reinforcement Learning (MERL) which is the backbone for MCE-IRL based algorithms as well as pre-existing Deceptive RL formulations. 

\paragraph{Markov Decision Process}

We consider environments that can be expressed as Markov Decision Processes (MDP). An MDP $M$ is defined by the tuple $(S, A, P, r, \gamma, \mu)$, where $S$ is the set of states, $A$ is the set of actions, $P(s'|s,a) \in [0,1]$ is the transition probability, $r(s,a) \in \mathbb{R}$ is the reward function, $\gamma \in [0,1]$ is the discount factor and $\mu$ is the initial state distribution. A policy $\pi(.|s)$ is a probability distribution over the set of valid actions for a given state. In this paper, we base our results on the assumption that both $S$ and $A$ are discrete, with every state reachable under $\mu$ and $P$. The cumulative $\gamma$-discounted value, or expected value, of the reward obtained by following $\pi$ in $M$ is denoted as $E_{\pi}[r(s,a)] = E[\sum_{t=0}^{\infty}\gamma^t r(s_t,a_t)]$. The occupancy measure $\rho_{\pi}: S\times A \rightarrow \mathbb{R}$ of a policy $\pi$ is defined as $\rho(s,a) = (1-\gamma)\pi(a|s)\sum_{t=0}^{\infty}\gamma^tP(s=s_t|\pi)$. The expected reward can be expressed in terms of occupancy measures as $E_{\pi}[r(s,a)] = \sum_s \sum_a \rho_{\pi}(s,a) r(s,a)$. For brevity, we sometimes use $\rho$ to denote the occupancy measure of a policy $\pi$. It is worth mentioning that there exists a one-one mapping between a policy and its corresponding occupancy measure \cite{10.1145/1390156.1390286}. For rest of this paper, we rely on this result to use $\pi$ and $\rho$ interchangeably. 




\paragraph{Deceptive Reinforcement Learning}

Let $R$ be the set of all reward functions, then a \textit{deceptive reinforcement learning} problem is defined by the tuple, $(S, A, P, r, \gamma, \mu, L^{\pi}_{R})$, where $S, A, P, r, \gamma, \mu$ are the same as defined for a regular MDP, and $L^{\pi}_{R}(s,a)$ stands for deception-inducted reward function \cite{topcu} which combines the objective of reward maximization with deception. 

A \textit{deceptive} policy maximises the objective, 
\begin{align}
    J_D = E_{\pi} [L^{\pi}_{R}(s, a)]
\end{align}

We do not make any assumptions about the knowledge of the observer similar to \cite{maxent}. In such a case, $L^{\pi}_{R}$ is a weighted mixture of the reward function and the deception level \cite{deceptiverl},
\begin{align}
\label{eq:drl}
    L^{\pi}_{R}(s,a) = \omega r(s,a) + d^{\pi}_{R}(s,a) 
\end{align}
where $d^{\pi}_{R}(s,a) \in \mathbb{R}$ is a measure of deception, and $\omega \in \mathbb{R}$ controls the trade-off between reward maximization and deception. For example, $d^{\pi}_{R}(s,a)$ could be the entropy of the policy, i.e, $-\log{\pi}(a|s)$, leading to deception by dissimulation.

\paragraph{Maximum Entropy RL (MERL)}
The objective of Maximum Entropy Reinforcement Learning (MERL) \cite{sac} is to optimize both the value function and the entropy or uncertainty of the agent's policy. Formally, 
\begin{equation}
\label{eq:merl}
    \text{RL}(r) = \underset{\pi}{\text{argmax}} \Hquad \Big\{ E_{\pi} \left[r(s,a) \right] + \mathcal{H}(\pi) \Big\}
\end{equation}
where $r$ is the reward function, $\mathcal{H}(\pi) \triangleq E_{\pi}[-\log(\pi(a|s))]$ is the $\gamma$-discounted cumulative casual entropy. MERL is an important algorithm in this paper as the MCE-IRL model (discussed below) is built on MERL. In addition, we showcase in subsequent sections that the MEIR algorithm (discussed below) is also an instance of MERL making it a "dissimulation" based Deceptive RL algorithm.

\paragraph{Maximum Entropy Intentional Randomization (MEIR)}
The MEIR algorithm is a private RL algorithm that is driven by the concept that maximizing the entropy of the policy $\mathcal{H}$ will keep the reward function private. MEIR solves the following optimization problem, 
\begin{align*}
    \text{MEIR}(r, E_{min}) = \Hquad &\underset{\pi}{\text{argmax}} \Hquad \mathcal{H}(\pi) \\ &\text{subject to} \Hquad E_{\pi} [r(s,a)] \geq E_{min} \numberthis \label{eq:meir}
\end{align*}
where $E_{min} \in [\hat{E}, E^*]$ is the reward threshold. $\hat{E} = E_{\hat{\pi}}[r(s,a)]$ and $E^* = E_{\pi^*}[r(s,a)]$ where $\hat{\pi}$ and $\pi^*$ correspond to the uniform random and optimal deterministic policies, respectively. The reward threshold $E_{min}$ is used to control the privacy/reward tradeoff. 

\paragraph{Maximum Causal Entropy IRL (MCE-IRL)}
The Maximum Causal Entropy Inverse Reinforcement Learning \cite{mceirl} model has emerged as the most prominent method to infer an unknown reward function from demonstrations. Given the occupancy measure $\tilde{\rho}$ of an agent (calculated from demonstrations) MCE IRL recovers a reward function based on the following formulation, 
\begin{equation*}
    \text{MCE-IRL}(\tilde{\rho}) = \Hquad \underset{r}{\text{argmax}} \Hquad \underset{\rho}{\text{min}} \Hquad E_{\tilde{\rho}}[r] - E_{\rho}[r] - \mathcal{H(\rho)} \label{eq:mceirl}
\end{equation*}

MCE-IRL is closely linked with MERL as the occupancy measures of the agent is obtained from the recovered reward function $\tilde{r}$ as,
\begin{equation}
   \label{eq:irl_rl_eq}
    \tilde{\rho} = RL(\tilde{r}) 
\end{equation}

\section{Assessing Learnt Reward, $\tilde{r}$, from IRL}
\label{sec:assesing}
Before we describe our contributions, we describe mechanisms to evaluate whether the original preferences (reward) are captured by the reward learnt by an observer (using IRL). It is challenging to assess the quality of the recovered reward function due to the inherent ambiguity in the Inverse RL problem. The ambiguity is on account of multiple reward functions being able to explain the demonstrated behaviour. To evaluate the quality of the recovered reward function, we adopt the framework proposed in \cite{skalse2023misspecification}, which introduces three quality standards. 

The first standard implies the preservation of the "ordering" of policies with respect to the true reward function, ${r}$ in the recovered reward function, $\tilde{r}$ as an indicator of the highest quality. A policy $\pi_1$ is better ($\succeq$) than a policy $\pi_2$ if the expected value with policy $\pi_1$ is higher than the expected value of $\pi_2$.  
\begin{align*}
\pi_1 & \succeq \pi_2 \iff \nonumber\\
 &E_{\pi_1}[{r}(s,a)] \geq E_{\pi_2}[{r}(s,a)] \land E_{\pi_1}[\tilde{r}(s,a)] \geq E_{\pi_2}[\tilde{r}(s,a)] 
\end{align*}

The second standard implies that the recovered reward function, $\tilde{r}$ shares the same set of optimal policies as the true reward function.  
$$\underset{\pi}{\text{argmax}} E_{\pi}[r(.,.)] = \underset{\pi}{\text{argmax}} E_{\pi}[\tilde{r}(.,.)] $$  

Lastly, the third criterion implies that the recovered reward function fails to preserve any of these desirable properties, indicating a lower level of quality in learning.

By applying these three quality standards, we can assess and compare the effectiveness of the recovered reward function. Matching optimal policies between the true and recovered reward functions signifies valuable insights into the agent's optimal trajectories. If policy ordering is preserved, the observer not only gains trajectory insights but also learns their order, increasing the chances of unveiling the agent's encoded preferences in the reward function.

Formally, Let $R$ be the set of all possible reward functions $r: S\times A \rightarrow \mathbb{R}$ with state space $S$ and action space $A$, and $M <S, A, P, _, \gamma, \mu>$ be an MDP without a reward function. Let partitions $OPT^M$ and $ORD^M$ be defined on $R$ as follows: given two reward functions $r_1$ an  $r_2$, we say that $r_1 \equiv_{OPT^M} r_2$ if $<S, A, P, r_1, \gamma, \mu>$ and $<S, A, P, r_2, \gamma, \mu>$ have the same set of \textit{optimal policies}, and $r_1 \equiv_{ORD^M} r_2$ if $<S, A, P, r_1, \gamma, \mu>$ and $<S, A, P, r_2, \gamma, \mu>$ have the same \textit{ordering over policies} \footnotemark. 

If two reward functions have the same \textit{ordering} of policies, then they have the same set of \textit{optimal} policies (Section 2.4 in \cite{skalse2023misspecification}).

\footnotetext{\textbf{Notation}: $x \equiv_{P} y$ denotes $x$ and $y$ belong to the same partition $P$.}

\section{Privacy Leak in MEIR}
\label{sec:meir}

We study the privacy leak of MEIR in two situations: (i) Observer has access to the agent's true occupancy measure; (ii) Observer obtains a few demonstrations instead of the true occupancy measure. For (i), we can theoretically prove that there exists a privacy leak. For (ii), we provide a bound on the quality of the recovered reward function $\tilde{r}$ w.r.t $r$ in terms of the distribution over policies they induce.

\subsection{True Occupancy Measure: MEIR $=$ MERL}
\label{sec:meir_true}
We show this by hypothesizing that any policy that is a solution for MEIR can be computed by solving the MERL problem, where the rewards are multiplied by a positive scalar. 

\begin{lemma}
\label{lambda_lemma}
 Any policy $\bar{\pi}$ that is the solution of a Max Entropy Intentional  Randomization formulation $\text{MEIR}(r, E_{min})$ with a reward constraint $E_{min} \in [\hat{E}, E^*]$, can be expressed as the solution of the Maximum Entropy RL problem as, 
 \begin{equation}
     \bar{\pi} = RL(\lambda^*r)
 \end{equation}
 for some $\lambda^* \geq 0$.
\end{lemma}

Lemma \ref{lambda_lemma} shows that MEIR algorithm implicitly solves the RL objective with an additional temperature parameter $\lambda$ (dual optimum) to intentionally control the trade-off between reward and entropy maximization. The RL objective can also be viewed as a belief inducted reward function $L^{\pi}_{R}$ of the form \ref{eq:drl}, where the measure of deception $d^{\pi}_{R}(s,a) = -\log{\pi(a|s)}$. As $\lim_{\lambda \to 0}$, the entropy term dominates the reward term and we obtain a uniform random policy. Similarly, as $\lim_{\lambda \to \infty}$, reward dominates the entropy term and we obtain the optimal deterministic policy. Proof for Lemma \ref{lambda_lemma} is provided in Appendix \ref{lambda_lemma_proof}. We now prove that MEIR suffers a privacy leak when used against an MCE-IRL-based observer.
\begin{theorem}
\label{leak_theorem}
For an MDP $M$ let $\bar{\pi} = \text{MEIR}(r, E_{min})$ for any reward constraint $E_{min} > \hat{E}$, and $\rho_{\bar{\pi}}$ be its corresponding occupancy measure. If $\tilde{r} = \text{MCE-IRL}(\rho_{\bar{\pi}})$ is the reward function recovered by an observer using Maximum Entropy IRL, then $r \equiv_{ORD^M} \tilde{r}$.
\end{theorem}

Theorem \ref{leak_theorem} states that the observer will recover a reward function that respects the \textit{ordering} of policies in the true reward function \textit{irrespective} of the reward threshold. This indicates that $\tilde{r}$ and $r$ share the same set of optimal policies as well. The proof for theorem \ref{leak_theorem} is built on the following lemma:
\begin{lemma} (Based on Theorem 3.4 in 
\cite{skalse2023misspecification}) 
\label{lemma_scale}
    For any two scalars, $\lambda_1, \lambda_2 \in \mathbb{R}^+$ and two reward functions $r_1, r_2$, if we have $RL(\lambda_1 r_1) = RL(\lambda_2 r_2)$, then $r_1 \equiv_{ORD} r_2$.
\end{lemma}

The proof for lemma \ref{lemma_scale} is given in Appendix \ref{lemma_scale_proof}. Lemma \ref{lemma_scale} states that if two reward functions yield the same policy when optimizing the objective \ref{eq:merl} with any positive weight assigned to the reward term, then the two reward functions have the same ordering over policies.

\begin{proof}[Proof of Theorem \ref{leak_theorem}]
Let $\bar{\pi} = \text{MEIR}(r, E_{min})$ be the randomized policy for any reward threshold $E_{min} > \hat{E}$. From Lemma \ref{lambda_lemma}, we know that $\bar{\pi} = RL(\lambda^* r)$, where $\lambda^* > 0$.  Let $\tilde{r} = \text{MCE-IRL}(\rho_{\bar{\pi}})$ be the reward function recovered by the observer using MCE-IRL. From Equation \ref{eq:irl_rl_eq}, we know that $\bar{\pi} = \text{RL}(\tilde{r})$. Hence,\\
$ \text{RL}(\lambda^*r) = \text{RL}(\tilde{r})
\implies r \equiv_{ORD^M} \tilde{r} \Hquad (\text{Lemma } \ref{lemma_scale}) $
\end{proof}

\subsection{Limited Demonstrations:}
\label{subsec:meir_probs}
In section \ref{sec:meir_true}, we showed that the policy generated by the MEIR algorithm can be represented as a solution to the MERL objective, i.e, $\text{MEIR}(r, E_{min}) = RL(\lambda^* r)$. The solution of the MERL objective can also be interpreted as a mixture policy over the set of all stochastic policies $\Pi$, with the weight given to each policy proportional it's value. Formally, a mixture policy ${\pi}^{mix}$ contains a set of policies $\{\pi_1, ..., \pi_n\}$, and a distribution $w$ over these policies. Before each episode, a policy is sampled according to $w$ and executed for the entire trajectory. The probability of a sampling policy $\pi \in \Pi$ is given by \cite{mceirl}, 
\begin{equation}
\label{eq:prop}
    P^{MERL}(\pi|r) \propto E_{\pi}[r]
\end{equation}
Thus, learning this distribution (or a reward function that \textit{induces} this distribution over policies) will give an observer insight into the ordering over policies in $r$. During execution however, from Lemma \ref{lambda_lemma}, we know that the agent samples a policy according to the distribution,
\begin{equation*}
    P^{MERL}(\pi|\lambda^* r) \propto E_{\pi}[\lambda^* r]
\end{equation*}
where $\lambda \geq 0$ which the observer learns via maximum likelihood (MCE-IRL). This distribution preserves the same ordering of policies in $r$ and hence the accurate estimation of this distribution by an observer poses a significant privacy leak. 

The quality of the distribution learned (Total Variation (TV) distance from the true distribution) is highlighted in Proposition \ref{prop:sample_complexity}. 
\begin{proposition} \label{prop:sample_complexity}

For any MDP $M$, let $\bar{\pi} = \text{MEIR}(r, E_{min})$ for any reward constraint $E_{min} > \hat{E}$, i.e, $\bar{\pi} = RL(\lambda r)$ for some $\lambda > 0$ (Lemma \ref{lambda_lemma}) and let $\rho_{\bar{\pi}}$ be the empirical occupancy measure of $n$ demonstrations obtained by executing $\pi$ in $M$. If $\tilde{r} = \text{MCE-IRL}(\rho_{\tilde{\pi}})$, then,
\begin{equation}
    Pr(TV(P^{MERL}(\pi|\lambda^* r), P^{MERL}(\pi|\tilde{r})) > \epsilon) \leq \delta
\end{equation}
where, $\delta = \Theta(e^{|\Pi|-n\epsilon^2})$
\end{proposition}

Proposition \ref{prop:sample_complexity} directly follows from \cite{tv_bound} given that $P(\pi|\tilde{r})$ is the maximum likelihood estimator of $P(\pi|\lambda r)$.

\begin{figure*}
    \centering
    \includegraphics[scale=0.5]{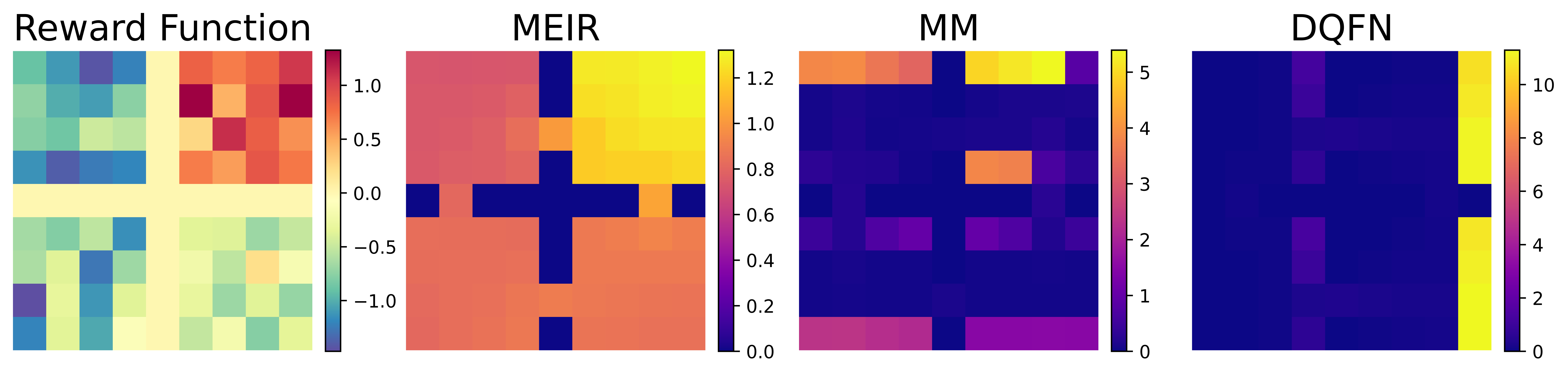}
    \caption{Occupancy measures of different private policies satisfying the same reward constraint in the Four Rooms environment. The MM algorithm leads to policies that visit a diverse mix of high reward and low reward states.}
    \label{fig:show_traj}
\end{figure*}

\section{The Max Misinformation Algorithm}
\label{sec:mm}
To address the privacy leak of MEIR, we introduce a novel algorithm referred to as the Max Misinformation (MM) algorithm. MM uses a deceptive measure called an anti-reward to incentivize the agent to \textit{stay away from} the optimal trajectories. That is to say, MM intentionally leads the agent to take sub-optimal trajectories, in a bid to fool the observer into believing that these trajectories are highly rewarding. This makes MM a \textit{simulation} based Deceptive RL algorithm.



Formally, let $r^{-}(s,a)$ be an anti-reward that induces the agent to take sub-optimal trajectories, then the MM algorithm solves the following constrained optimization problem, 
\begin{align*}
    \text{MM}(r, r^-, E_{min}) = \Hquad &\underset{\pi}{\text{argmax}} \Hquad E_{\pi} [r^{-}(s,a)] \Hquad \\ &\textbf{s.t.} \Hquad E_{\pi} [r(s,a)] \geq E_{min} \numberthis \label{eq:mm}
\end{align*}
where $E_{min} \in [E^{-}, E^*]$ is the reward threshold that is used to control the privacy-expected reward tradeoff. $E^{-} = E_{\pi^{-}}[r(s,a)]$ and $E^* = E_{\pi^*}[r(s,a)]$ where $\pi^{-}$ and $\pi^{*}$ correspond to the optimal policies with respect to the anti-reward, $r^{-}$ and actual reward, $r$.  We describe mechanisms for computing anti-reward in Section~\ref{subsec:gen_anti_r}.

The MM formulation is a linear program (Appendix \ref{eq:mm_lp}) and hence the primal optimum $\bar{\pi}$ can be uniquely recovered from the dual optimum $\lambda^*$ \cite[Section 5.5.5]{DBLP:books/cu/BV2014} as, 
\begin{equation}
\label{eq:mm_dual}
    \bar{\pi} = \underset{\pi}{\text{argmax }}E_{\pi}[\lambda^* r(s,a) + r^{-}(s,a)]
\end{equation}
Equation \ref{eq:mm_dual} is a form of the Deceptive RL objective \ref{eq:drl}, where $d^{\pi}_{R}(s,a) = r^-(s,a)$ and the dual variable $\lambda^*$ acts as a temperature parameter controlling the trade-off between reward and deception maximization. As $\lim_{\lambda^* \to 0}$, the anti-reward dominates the reward resulting in $\pi^-$, and as $\lim_{\lambda^* \to +\infty}$ the reward dominates the anti-reward resulting in $\pi^*$ as the solution to Equation \ref{eq:mm_dual}.

The linear program formulation of MM presented in Appendix \ref{eq:mm_lp} relies on known model dynamics. However, this assumption is often impractical in real-world scenarios. To address this limitation, we introduce Algorithm \ref{alg:mm_pd_descent}, which demonstrates how to solve Equation \ref{eq:mm} using primal-dual descent without requiring explicit knowledge of the model dynamics. This formulation is particularly useful in the context of large MDPs with continuous state and (or) action spaces. In such scenarios, solving the primal problem to convergence (Line \ref{conv_line} of Algorithm \ref{alg:mm_pd_descent}) can be very time-consuming. Instead, one could take a few steps towards maximizing the objective function: Line~\ref{conv_line} and then incrementally optimize $\lambda$ and so on. 

\begin{algorithm}
\caption{Max Misinformation via Primal-Dual Descent}\label{alg:mm_pd_descent}
\begin{algorithmic}[1]
\State \textbf{Input:} Anti-reward $r^{-}$, $E_{min} \in [E^-, E^*]$
\State Initialize temperature parameter $\lambda >= 0$, learning rate $\alpha$
\For{t in \{0, 1, 2, $\ldots$\}} 
    \State \label{conv_line} $\bar{\pi}$ = max $E_{\pi}[\lambda_t r(s,a) + r^-(s,a)]$ 
    \State $\lambda_{t+1} \leftarrow \lambda_{t} - \alpha \nabla_{\lambda} \left[\lambda [E_{\bar{\pi}} r(s,a) - E_{min}]\right]$ 
\EndFor
\end{algorithmic}
\end{algorithm}

In the case of discrete MDPs, where solving the primal problem is much faster, we can use binary search to speed up the optimization procedure significantly as highlighted in Appendix \ref{app:binary_search}. 

\subsection{Security of the Max Misinformation Algorithm}

In Section \ref{sec:meir}, we highlighted that the privacy leakage in the MEIR algorithm was due to the fact that MEIR is an instance of MERL, which preserves the ordering of policies which can be learnt efficiently by the observer. The proposed MM algorithm addresses this privacy leakage as the addition of an anti-reward does not preserve the ordering over policies as it assigns a high value to sub-optimal trajectories. Consequently, it does not preserve the set of optimal policies either.

The difference between the occupancy measures of the MEIR and the MM algorithms are highlighted in Figure \ref{fig:show_traj}. Despite the MEIR policy exhibiting higher entropy, it readily reveals locations with high rewards. In contrast, the MM policy visits a nuanced mix of high and low reward states, making it more challenging to discern important locations.
 
\subsection{Generating anti-reward functions} 
\label{subsec:gen_anti_r}
We now describe mechanisms for generating anti-reward functions that maximize deception by steering an agent away from the optimal trajectories.

Ideally, we would like to maximize the distance between the true reward function and the recovered reward function, but this would make the deceptive policy specific to an IRL algorithm (as recovered reward function is dependent on the algorithm). Instead, to ensure robustness against the observer reward recovery methods (IRL or some other mechanism), we propose a mechanism that is agnostic to the specific algorithm utilized to recover the reward function. 


Intuitively, we compute an anti-reward that maximizes the distance between a distribution/statistic corresponding to the optimal policy for the original reward and optimal policy for the anti-reward. This will ensure that observer receives minimal information about the optimal policy for the original reward.
Let $o$ be a distribution/statistic that can be computed from the agent's reward function $r$. Two examples of $o$ would be the policy computed or occupancy distribution corresponding to the reward function $r$. We can generate an anti-reward function $r^{-}$ by maximizing the distance between $o^{*}$ and $o^{-}$, where $o^{*}$ is observed when behaving optimally according to $r$ and $o^{-}$ is observed when behaving optimally according to $r^{-}$. The algorithm for computing the anti-reward is provided in Algorithm \ref{alg:gen_anti_reward}.  Let $C$ be the function that maps from $r$ to $o$. We iteratively do the following steps by starting from a randomly initialised $o^-$: (1) Set $r^{-}$ as the anti-reward function that maximises the distance between $o^{*}$ and $o^{-}$ (2) Compute the new $o^{-}$ from the new value of $r^{-}$. We repeat this process for a set number of iterations. An intuition behind why this approach works is highlighted in Appendix \ref{app:idea_anti}. 

\begin{figure*}
\centering
\includegraphics[width=2\columnwidth]{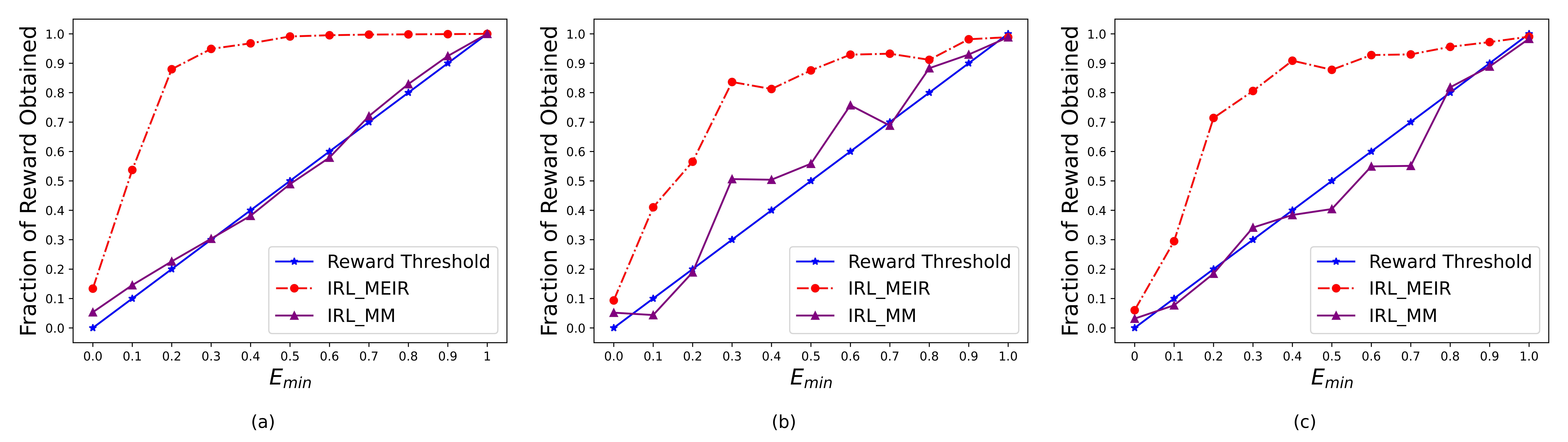}
\caption{MM and MEIR against IQ-Learn given $10$ demonstrations. Figures correspond to (a) Cyber security domain, (b) Frozen Lake and (c) Random MDPs.}
\label{fig:comp_meir_iq}
\end{figure*}

\begin{algorithm}
\caption{Generating Anti-Reward functions}\label{alg:gen_anti_reward}
\begin{algorithmic}[1]
\State \textbf{Input:} Distance Metric $D$, $c$, $o^{*}$ 
\State Initialize $o^{-}$
\For{t in \{0, 1, 2, $\ldots$\}} 
    \State $r^{-} = \underset{r}{\text{argmax}} \Hquad D(o^{*}, o^{-})$ \label{div_line}
    \State $o^{-} = C(r^{-})$
\EndFor
\end{algorithmic}
\end{algorithm}

We will now outline the various forms of $o$ (occupancy measures and trajectory distributions) and the corresponding distance measures, denoted as $D$, applied in each case.

\paragraph{Occupancy Measures}
Since IRL algorithms try to match occupancy measures, one could try to directly maximize the distance between the occupancy measures of the optimal policy of $r$, i.e, $\rho^*$ and $\rho^{-}$. Hence, in this case, $o = \rho$. We use $f$-divergences and Integral Probability Metrics (IPMs) to measure the distance between $\rho^*$ and $\rho^{-}$. 

The $f$-divergence between $\rho$ and $\rho^{*}$ is defined using the convex conjugate $f^*$ as, 
\begin{align*}
    D_{f}(\rho^{*}||\rho^-) &= \Hquad \underset{g: D \rightarrow R}{sup} \Hquad E_{\rho^{*}}[g(s,a)] - E_{\rho^{-}}[f^*(g(s,a))] 
\end{align*}
setting $g = -r^{-}$ and $\phi(u) = -f^*(-u)$, so that $E_{\rho_-}$ is maximized, 
\begin{align}
\label{eq:f_div}
D_{f}(\rho^{*}||\rho^-) &= \Hquad \underset{r^-: D \rightarrow R}{sup} \Hquad E_{\rho_{-}}[\phi(r^-(s,a))]] - E_{\rho^{*}}[r^{-}(s,a)]
\end{align}

IPMs that are parameterized by a family of functions $\mathcal{F}$ are defined,  
\begin{align}
\label{eq:ipm}
    \gamma_{\mathcal{F}}(\rho^{-}, \rho^*) = \underset{f \in \mathcal{F}}{sup} \Hquad |E_{\rho^{-}}[f(s,a)] - E_{\rho^{*}}[f(s,a)]|
\end{align}

We can see that in both IPMs  (Equation \ref{eq:ipm}) and f-divergences (Equation \ref{eq:f_div}), the anti-reward function gives a high reward to the state-action pairs visited by $\pi^{-}$ (due to the $sup$ and expectation over $\rho^-$ being the first term) and a low reward to the ones visited by $\pi^*$. In both these cases, the anti-reward function can be represented using a function approximator and solved using gradient ascent, or in the case of discrete environments, Equation \ref{eq:f_div} can be solved using the closed form solution (Table \ref{table:closed_f_div} in Appendix \ref{app:idea_anti}).


\paragraph{Trajectory Distributions}
We can observe from Equation \ref{eq:prop} and Figure \ref{fig:show_traj} that the MEIR algorithm suffers from the problem of preferring highly rewarding trajectories (policies in the case of stochastic dynamics). To avoid this, we can generate an anti-reward function that maximizes the distance between the trajectory distributions of $\pi^*$ and $\pi^{-}$. Consider the objective of maximizing the KL-divergence, between the distribution over trajectories induced by a policy $\pi$ and the optimal policy $\pi^{*}$,
\begin{equation}
\begin{aligned} 
    J^- &= \Hquad \underset{\pi}{\text{argmax}} \Hquad KL\left(p(\tau)||p(\tau^*)\right) \\
        &= \Hquad \underset{\pi}{\text{argmax}} \Hquad \int p(\tau) \log\left(\frac{\prod_{t} p(s_{t+1}|s_t, a_t)\pi(a_t|s_t)}{\prod_{t} p(s_{t+1}|s_t, a_t)\pi^*(a_t|s_t)}\right) d\tau\\
        &= \Hquad \underset{\pi}{\text{argmax}} \Hquad -\mathcal{H}(\pi) + E_{\tau\sim\pi}[\sum_{t=0}^T-\log{\pi^*(a_t|s_t)}] \\
        &= \Hquad \underset{\pi}{\text{argmax}} \Hquad E_{\tau\sim\pi}[r^{-}_{KL}(s_t, a_t)]
\end{aligned}
\end{equation}
This formulation is the standard RL objective with the anti-reward function $r_{-}^{KL} \triangleq -\log(\pi^*(a|s))$ \footnotemark. We can drop the entropy term as the formulation calls for entropy \textit{minimization} given that an MDP has an optimal deterministic solution, i.e., a policy with 0 entropy.

\footnotetext{If $\pi^*$ is deterministic $\log(\pi^*(a|s))$ is not defined when $\pi^*(a|s) = 0$. This can be avoided by setting $\pi^*$ as the a solution of MERL which ensures all actions have non-zero support.}

\section{Experiments and Results}
\label{sec:exps}

In this section, we intend to answer the following key questions: (1) Does the MEIR algorithm as described in Section \ref{sec:meir} suffer a significant privacy leak in the case of limited demonstrations, and how does the MM algorithm perform in comparison? (2) How does the MM algorithm fare in comparison to the MEIR and DQFN algorithms in preserving the privacy of the reward function when the observer has access to the true occupancy measures of the agent? (3) How does MM perform against observers that know they are being deceived?

\subsection{Environments and Evaluations Metrics}

We conduct our experiments in the following environments: Cyber Security which is based on Moving Target Defence \cite{mtd}, Frozen Lake \cite{gym}, Four Rooms \cite{minigrid} and randomly generated MDPs. In the Cyber Security \cite{mtd} environment, the state space consists of network configurations generated by the CyberBattleSim library \cite{msft:cyberbattlesim} that emulates real-world active directory networks. The value associated with a network configuration corresponds to the ease with which malicious actors could compromise and gain control over the network. Agent's objective is to dynamically switch between network configurations to bolster security. The intention behind using Frozen Lake \cite{gym} and Four Rooms \cite{minigrid} from the standard Gym environment draws inspiration from tangible real-world security challenges such as Police Patrolling \cite{chen2013police} and Green Security Games (GSGs) \cite{GSG_1}. Randomly generated MDPs can be representative of a broad set of domains in general. A more detailed description is provided in Appendix \ref{app:envs}. The reward function in the above domains reflects the user's preferences, encompassing critical aspects such as the valuation of patrol locations, density of animals in various regions, and significance of distinct network configurations, all of which must be kept private.

We consider Inverse Reinforcement Learning (IRL) as the main method for reward function reconstruction due to two primary considerations: (a) Observer does not know that they are being deceived, which is the assumption made in the prior deception based methods as well \cite{deceptiverl, deceptiverl2} (b) Recovering the reward function from deceptive demonstrations is not a trivial task. Consider the example in Figure \ref{fig:show_traj} - the agent visits multiple diverse locations and deducing the high reward states from them is not easy for the observer. In addition, there does not exist any prior work in this space to benchmark our algorithms against. In such a case, irrespective of whether the observer knows that they are being deceived, learning a reward function that maximises the likelihood of the demonstrations is a strong strategy for the observer given that they are guaranteed at least $E_{min}$ reward. 

In the IRL space, MCE-IRL based methods have demonstrated superior performance in recent times for reward reconstruction \cite{gail, iql} and hence we use MCE-IRL and IQ-learn \cite{iql} (a variant of MCE-IRL that performs well in the limited demonstrations setting) in our experiments. Furthermore, we introduce two additional baselines that try to account for observers that are aware that they are being deceived by the usage of the MM algorithm. Given that the observer knows that the agent is intentionally visiting sub-optimal states along with the optimal states, the observer can cluster the occupancy measures of the agent and selectively recover a reward function that matches just one (or more) of them. See Figure \ref{fig:show_traj} and \ref{fig:comp_modes} - each cluster either contains the optimal states or the misleading states intentionally visited by the agent. Based on how we select the cluster(s), we split these methods into: (a) $\text{IRL}^{\text{random}}$ that picks a cluster at random and (b) $\text{IRL}^{\text{max}}$ that greedily picks the cluster that has the highest occupancy measure.

We use three metrics to evaluate the quality of the reward function learned by the IRL algorithms, namely: (1) Pearson Correlation: high value indicates better learning; (2) EPIC distance \cite{epic}: low value indicates better learning); and (3) Evaluation of the optimal policy of the recovered reward function in the original reward function: higher expected reward implies better learning. 

\subsection{Analysis of Results}
\paragraph{Privacy Leak in MEIR and the efficacy of MM}

Figures \ref{fig:comp_meir_iq}, \ref{fig:comp_meir_mce} and \ref{fig:comp_deceptive_irl} can be interpreted as follows: for a given private RL algorithm. First, we specify a reward constraint $E_{min}$ that is represented on the x-axis. Next, the algorithms return a private policy with a return $\geq E_{min}$ that is plotted on the y-axis. The Reward Threshold line indicates the return of the generated policy. Post this, an observer takes the occupancy measure of the private policy as input and recovers a reward function $\tilde{r}$. The return of the optimal policy $\pi^{*}_{\tilde{r}}$ of $\tilde{r}$ when evaluated in $r$ is represented using the IRL line. 

From Figure \ref{fig:comp_meir_iq}, we can infer that the reward recovered by $\pi^{*}_{\tilde{r}}$ is very high, almost the same as the optimal policy $\pi^*_{r}$, even when $E_{min}$ is very low and when the observer is given just $10$ trajectories to learn from. Therefore validating the insights presented in Section \ref{subsec:meir_probs}. MM algorithm does a much better job in preserving the reward privacy in this case. Figure \ref{fig:comp_meir_mce} shows that MM significantly outperforms MEIR in the worst case, i.e when an observer has access to the true occupancy measures of the agent. A quantitative analysis for this case is present in Table \ref{table:comp_meir_app} in Appendix \ref{app:comp_meir}.

\paragraph{Advantage over DP-based methods}

\begin{figure}
\includegraphics[width=\columnwidth]{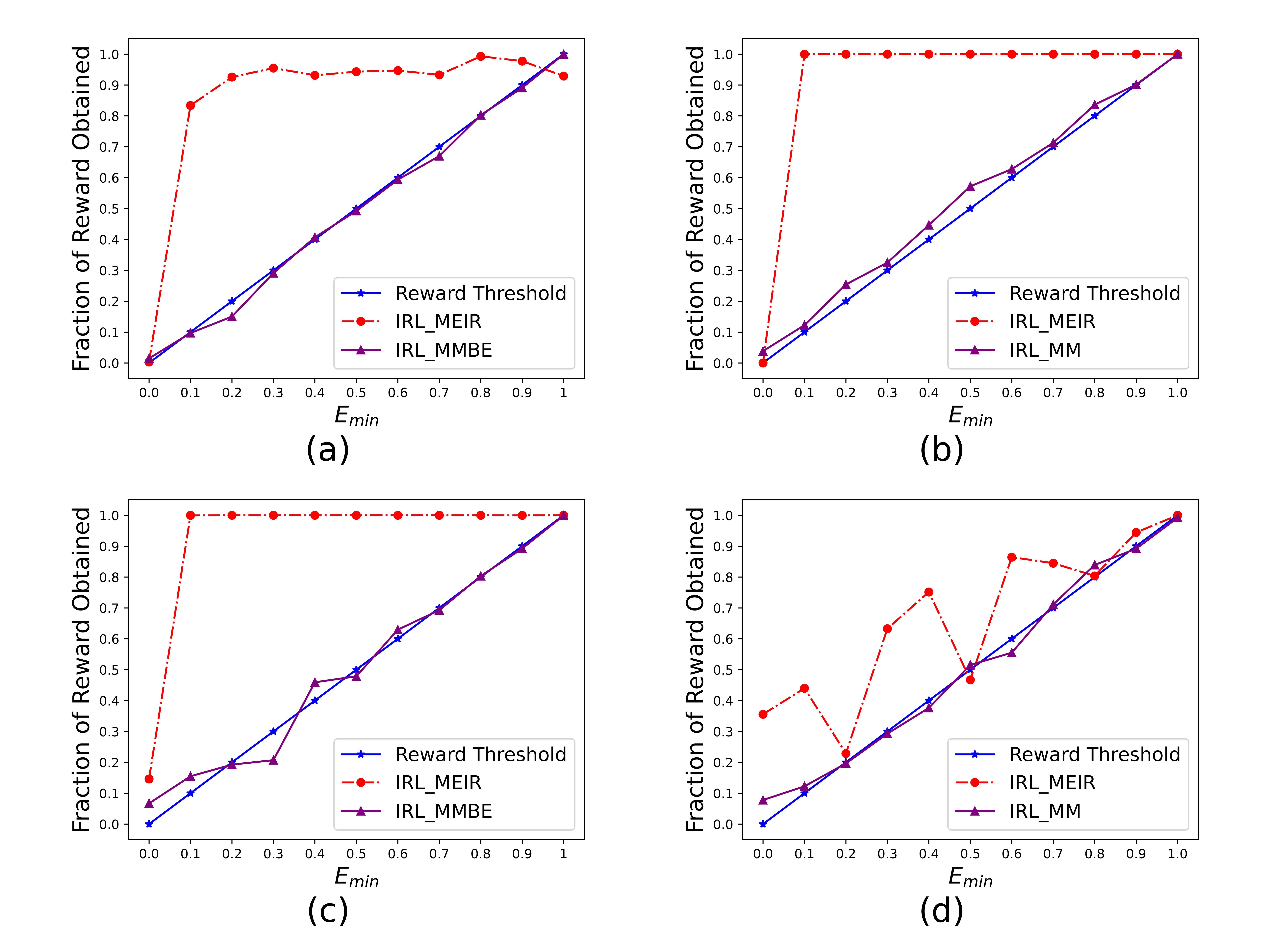}
\caption{MM and MEIR against MCE IRL with true occupancy measures. Figures correspond to (a) Four Rooms (b) Frozen Lake (c) Random MDPs (d) Cyber Security domain.}
\label{fig:comp_meir_mce}
\end{figure}

\begin{table}[t]
\caption{MM and the DQFN algorithms against an MCE IRL based observer with access to the true occupancy measure.}
\label{table:comp_dqfn}
\begin{adjustbox}{width=\columnwidth}
\begin{tabular}{|llcc|}
\hline
\textbf{Env} & \textbf{Algorithm} & \textbf{Avg. Pearson} & \textbf{Avg. EPIC} \\
\hline
\multirow{3}{*}{Random MDP} & DQFN & 0.31 & 0.58\\
&  MM ($\tau$-based) & \textbf{0.30} & \textbf{0.58}\\
 &  MM ($\rho$-based) & 0.67 & 0.39\\
\hline
\multirow{3}{*}{Four Rooms} & DQFN & 0.09 & 0.67\\
&  MM ($\tau$-based) & \textbf{0.03} & \textbf{0.69}\\
&  MM ($\rho$-based) & 0.22 & 0.62\\
\hline
\multirow{3}{*}{Frozen Lake} & DQFN & 0.11 & 0.67\\
&  MM ($\tau$-based) & \textbf{0.05} & \textbf{0.72}\\
&  MM ($\rho$-based) & 0.34 & 0.57\\
\hline
\end{tabular}
\end{adjustbox}
\end{table}

\begin{figure}
\includegraphics[width=\columnwidth]{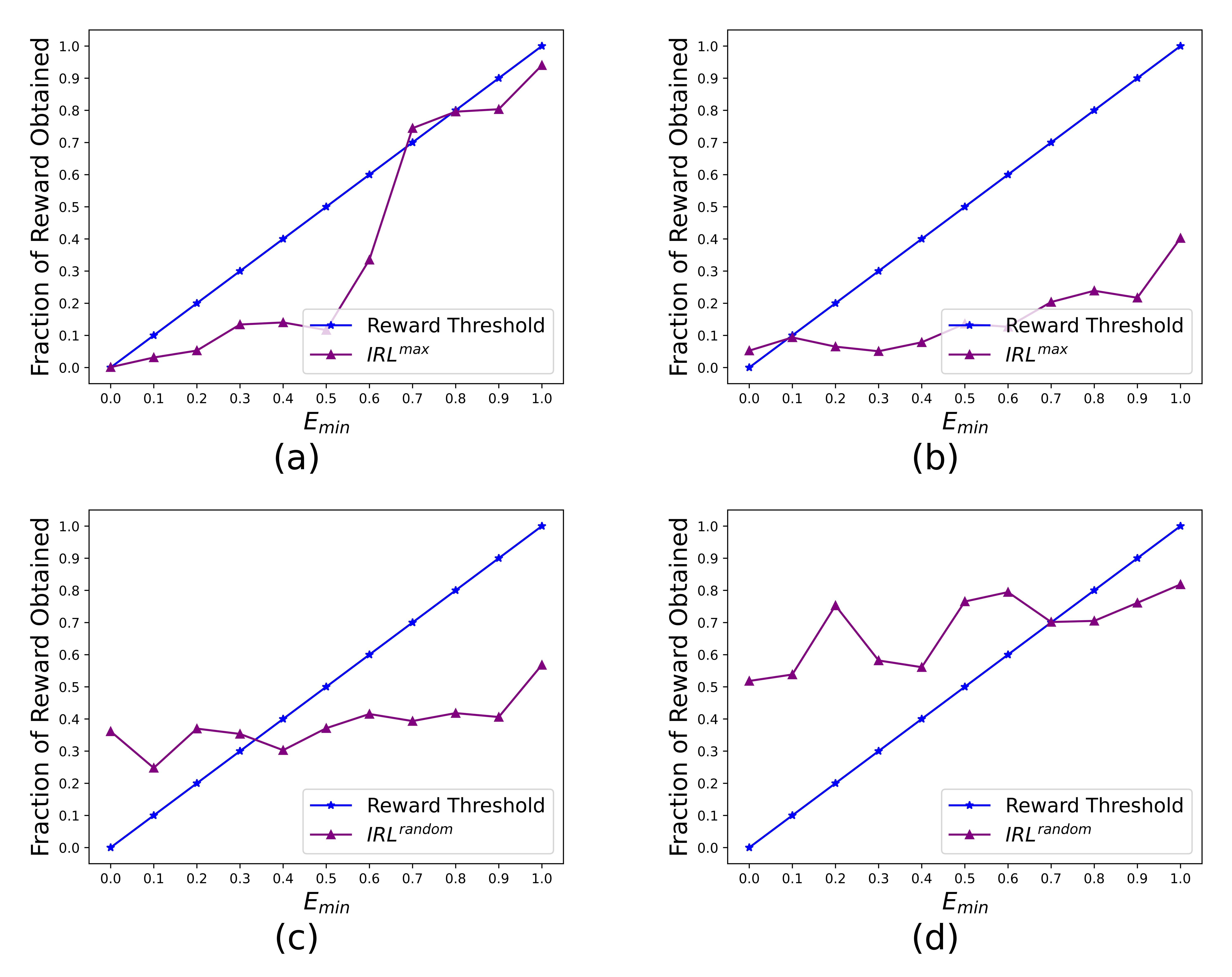}
\caption{MM against $\text{IRL}^{\text{max}}$ in (a) Four Rooms (b) Frozen Lake and $\text{IRL}^{\text{random}}$ in (c) Random MDP (d) Cyber Security domain.}
\label{fig:comp_deceptive_irl}
\end{figure}

The quantitative comparison between the Deep Q-learning with Functional Noise (DQFN) \cite{dqfn}, a state-of-the art algorithm in DP-based privacy methods and the MM algorithm in Table \ref{table:comp_dqfn} demonstrates that the MM algorithm with a trajectory-based anti-reward outperforms the DQFN algorithm. Additionally, a qualitative comparison (Figure \ref{fig:show_traj}) shows that the MM algorithm visits a significantly more diverse set of misleading states compared to the DQFN algorithm. In addition, as discussed in Section \ref{sec:intro}, the MM algorithm has an advantage over existing DP-based methods based on the following: (1) There is a direct relationship between the privacy budget $E_{min}$ and the reward obtained by the corresponding policy. The noise parameter in DQFN has no such relationship (Figure \ref{fig:dp_rew}), making it challenging to control the privacy/reward tradeoff and (2) The quality of the recovered reward function is proportional to the privacy budget $E_{min}$ as seen in Figures \ref{fig:comp_meir_iq} and \ref{fig:comp_meir_mce}, another important property that DP-based algorithms do not posses \cite{Prakash_Husain_Paruchuri_Gujar_2022}.

\paragraph{Effectiveness against observers who are aware of the use of deception by the agent}
Figure \ref{fig:comp_deceptive_irl} contains evaluations of MM against an observer that is aware of the use of deception and subsequently uses $\text{IRL}^{\text{max}}$ / $\text{IRL}^{\text{random}}$ to recover the reward function. From Figure \ref{fig:comp_deceptive_irl}, we can infer (1) robustness of the MM algorithm in preserving the privacy of the reward function in this scenario and (2) weakness of these reward recovery mechanisms as compared to IRL indicating the difficulty of reward recovery when faced with deceptive demonstrations and the dominance of IRL as a strategy for the observer.

\section{Conclusion and Future Work}

RL-based planning algorithms have found applications in many domains including security related where protection of the reward function from potential observers becomes critical. Our study identifies vulnerabilities and limitations in existing methodologies and proposes the Max Misinformation (MM) algorithm as a solution. While our experiments are limited to discrete state/action spaces, MM can be used in continuous settings, making for future research endeavors. Furthermore, our research underscores the limitations of Inverse Reinforcement Learning (IRL) when confronted with deceptive demonstrations, prompting further exploration into Deceptive IRL.

\section{Acknowledgments}
This research/project is supported by the National Research Foundation Singapore and DSO National Laboratories under the AI Singapore Programme (AISG Award No: AISG2-
RP-2020-016)

\bibliography{main.bib}

\appendix

\clearpage

\section{Proofs}
\subsection{Proof of Lemma \ref{lambda_lemma}}
\setcounter{lemma}{0}
\begin{lemma}
 Any policy $\bar{\pi}$ that is the solution of a Max Intentional Entropy randomization problem $\text{MEIR}(r, E_{min})$ with a reward constraint $E_{min} \in [\hat{E}, E^*]$, can be expressed as the solution of the Maximum Entropy RL problem as, 
 \begin{equation}
     \bar{\pi} = RL(\lambda^*r)
 \end{equation}
 for some $\lambda^* \geq 0$.
\end{lemma}

Before we begin the proof, we first introduce some of the properties of occupancy measures. Any valid occupancy measure $\rho$ must satisfy a set of affine constraints known as the Bellman Flow Constraints \cite{10.1145/1390156.1390286}: 
$$\rho(s,a) \geq 0, \forall (s,a) \in S\times A, \text{ and }$$ 
$$\sum_a \rho(s,a) = \mu_s + \gamma\sum_{s'}\sum_a P(s'|s,a)\rho(s',a), \forall s \in S$$ If $\Pi$ is the set of all stationary policies, then let $D = \{\rho_{\pi}: \pi \in \Pi\}$ be the set of all valid occupancy measures. There exists a bijection between and $D$ and $\Pi$ \cite{10.1145/1390156.1390286}. Hence, $\pi$ can be uniquely recovered from $\rho$ as $$\pi(a|s) = \frac{\rho(s,a)}{\sum_{a'} \rho(s,a')}$$ The cumulative $\gamma$-discounted causal entropy $\mathcal{H}(\pi)$ can also be defined using occupancy measures~\cite{gail}: $$\bar{\mathcal{H}}(\rho) \triangleq \sum_s \sum_a \rho(s,a) \log\left(\frac{\rho(s,a)}{\sum_{a'}\rho(s,a')}\right)$$   

\begin{proof}
\label{lambda_lemma_proof}
    The MEIR algorithm can be written in terms of occupancy measures as, 
\begin{align}
    \text{MEIR}(r, E_{min}) = \Hquad &\underset{\rho \in D}{\text{argmax}} \Hquad \bar{\mathcal{H}}(\rho) \Hquad \\ 
    &\text{subject to} \Hquad \sum_s \sum_a \rho(s,a) r(s,a) \geq E_{min}
\end{align} 
The resulting Lagrangian can be expressed as, 
\begin{equation*}
    L(\rho, \lambda) = \bar{\mathcal{H}}(\rho) + \lambda \left[ \sum_s \sum_a  \rho(s,a) r(s,a) - E_{min} \right]
\end{equation*}

The feasible set $D' = \{\rho: \rho \in D \text{ and } \rho^Tr \geq E_{min}\}$ is a convex set (as it is defined by a set of affine constraints) and $\bar{\mathcal{H}}$ is strictly concave \cite{gail}, hence strong duality holds, and the primal optimal 
$$\rho^* = \underset{\rho \in D}{\text{argmax }} \underset{\lambda}{\text{min }} L(\rho, \lambda)$$ can be uniquely recovered from the dual optimum $$\lambda^* = \underset{\lambda}{\text{argmin }} \underset{\rho \in D}{\text{max }} L(\rho, \lambda) \text{, as}$$  
$$\rho^* = \underset{\rho \in D}{\text{argmax }} L(\rho, \lambda^*) = RL(\lambda^*r)$$ (Section 5.5.5 in \cite{DBLP:books/cu/BV2014}). Hence, the randomized policy $\bar{\pi}$ can be obtained as,
\begin{equation}
    \bar{\pi} = RL(\lambda^*r)
\end{equation}
\end{proof}


\subsection{Proof of Lemma \ref{lemma_scale}}
\setcounter{theorem}{0}
\setcounter{lemma}{1}
\setcounter{proposition}{0}
Lemma \ref{lemma_scale} is a corollary of Theorem 3.4 in \cite{skalse2023misspecification}. Theorem 3.4 in \cite{skalse2023misspecification} defines all the functions $g$ such that if $g(r_1) = RL(r_2)$ for two reward functions $r_1, r_2$, then they share the same ordering over policies. We use this Theorem to show that $RL(\lambda \cdot r_1)$ where $\lambda \in \mathbb{R^+}$ belongs to this class of functions, i.e, if $RL(\lambda \cdot r_1) = RL(r_2)$, then $r_1$ has the same ordering of policies as $r_2$.


Prior to presenting this theorem and proving Lemma \ref{lemma_scale}, we introduce the necessary notation and definitions. 

\begin{enumerate}
    \item Let $R$ be the set of all possible reward functions $r: S\times A \rightarrow \mathbb{R}$. A reward object $f: R \rightarrow X$ is a mapping from $R$ to an arbitrary set $X$.  In this section we are only concerned with reward objects of the form $R \rightarrow \Pi$, i.e, the MERL, MEIR and MM algorithms.
    \item Given a partition $P$ of $R$, we say that a reward object, $f$ is $P$-\textit{admissible} if $f(r_1) = f(r_2) \Rightarrow r_1 \equiv_{P} r2$ 
    \item \label{def:probust}
    Given a partition $P$ of $R$, $f$ is said to be $P$-\textit{robust} to \textit{misspecification} with $g$, if $f$ is $P$-\textit{admissible}, $f \neq g$ and $f(r_1) = g(r_2) \Rightarrow r_1 \equiv_P r_2$.
\end{enumerate} 

Let $\psi: r \rightarrow \mathbb{R}^+$ represent a function for weighing the rewards. Let $c_{\psi}(r) \triangleq \text{RL}(\psi(r) \cdot r)$ denote objective \ref{eq:merl} with a weight of $\psi(r)$ given to the reward term and $C = \{c_{\psi}: \psi \in  r \rightarrow \mathbb{R}\}$ be the set of all such functions. 

\begin{proposition} (Theorem 3.4 in \cite{skalse2023misspecification})
\label{ref_prop_1}
    Consider $c_{\psi} \in C$,  then, $c_{\psi}$ is $ORD^M$-robust to misspecification with $g$ if and only if $g \in C$ and $g \neq c_{\psi}$.
\end{proposition} 

We would kindly direct the reader to \cite{skalse2023misspecification} (Section A.3) for the proof of Proposition \ref{ref_prop_1}.

\begin{lemma} (Corollary of Proposition \ref{ref_prop_1})
 For any two scalars, $\lambda_1, \lambda_2 \in \mathbb{R}^+$ and two reward functions $r_1, r_2$, if $RL(\lambda_1 r_1) = RL(\lambda_2 r_2)$, then $r_1 \equiv_{ORD^M} r_2$.
\end{lemma}
\begin{proof}
\label{lemma_scale_proof}
$RL(\lambda_1 r_1)$ and $RL(\lambda_2 r_2)$ can be written as $c_{\psi_1}$ and $c_{\psi_2}$ respectively where $\psi_1(r) = \lambda_1$ and $\psi_2(r) = \lambda_2 \Hquad \forall r \in R$. We break the proof down into two cases:

\textit{Case 1: $\psi_1 \neq \psi_2$} \\Proposition \ref{ref_prop_1} states $\psi_1$ is $ORD^M$-\textit{robust} to \textit{misspecification} with $\psi_2$, hence, from the definition $P$-\textit{robustness} we get  $\psi_1(r_1) = \psi_2(r_2) \Rightarrow r_1 \equiv_{ORD^M} r_2$.
 
\textit{Case 2: $\psi_1 = \psi_2$} \\
For $c_{\psi}$ to be $ORD^M$-\textit{robust} to \textit{misspecification}, it must be $ORD^M$-\textit{admissible} (Definition \ref{def:probust}). From the definition of $P$-\textit{admissibility}, we get $\psi_1(r_1) = \psi_1(r_2) \Rightarrow r_1 \equiv_{ORD^M} r_2$.

\end{proof}

\section{Linear Program Formulation of the Max Misinformation Algorithm}
\label{eq:mm_lp}
The Max Misinformation algorithm can be written as a linear program in terms of its occupancy measures,
\begin{equation}
    \text{MM}(R, E_{min}) = \Hquad \underset{\rho}{\text{argmax}} \Hquad \sum_s \sum_a \rho(s,a) r^-(s,a) 
\end{equation}
subject to, 
\begin{align*}
    \rho(s,a) \geq 0 \Hquad \forall s \in S, a \in A \\
    \sum_a \rho(s,a) = \mu_s + \gamma\sum_{s'}\sum_aP(s'|s,a)\rho(s',a) \Hquad \forall s \in S, a \in A \\
    \sum_s \sum_a \rho(s,a) r(s,a) \geq E_{min}
\end{align*}

\section{Solving the MM formulation}

\subsection{Binary Search solution for MM}
\label{app:binary_search}
Let $\bar{E}$ be the expected reward obtained by a policy $\bar{\pi}$ maximizing the objective in Line \ref{conv_line} of Algorithm \ref{alg:mm_pd_descent} for an arbitrary $\lambda \geq 0$. It should be noted that as $\lambda$ increases (decreases), $\bar{E}$ increases (decreases). Due to this property, we can search for the optimal $\lambda^*$ via binary search, which drastically improves runtime. This is presented in Algorithm \ref{alg:mm_binary_search}.

\begin{algorithm}
\caption{Max Misinformation via Binary Search}\label{alg:mm_binary_search}
\begin{algorithmic}[1]
\State \textbf{Input:}  Anti-reward $r^{-}$, $E_{min} \in [E^-, E^*]$, $\lambda_{max}$, $\epsilon$
\State Initialize temperature parameter $i = 0, j = \lambda_{max}$
\While{$|\bar{E} - E_{min}| < \epsilon$}
    \State$ \lambda = (i+j)/2$
     \State \label{conv_line2}$\bar{\pi}$ = max $E_{\pi}[\lambda r(s,a) + r^-(s,a)]$ 
     \State $\bar{E} = E_{\bar{\pi}}[r(s,a)]$
     \If{$\bar{E} < E_{min}$} $i = \lambda$
     \Else \text{  } $j = \lambda$ 
     \EndIf
\EndWhile

\end{algorithmic}
\end{algorithm}

\subsection{Solving the MEIR formulation}
The original formulation of the MEIR algorithm, as presented in \cite{maxent}, relies on a model-based approach limited to discrete environments. To transform the MEIR algorithm into a model-free version, one can employ primal-dual descent or binary search techniques by replacing the contents of Lines \ref{conv_line} and \ref{conv_line2} in Algorithm \ref{alg:mm_pd_descent} and \ref{alg:mm_binary_search} with $\bar{\pi} = E_{\pi}[\lambda_i r(s,a) - \log(\pi(a|s))]$. The Soft-Q learning algorithm described in \cite{mceirl} can be used to solve this expectation. In situations involving continuous state spaces, extensions like Soft Actor-Critic \cite{sac} offer practical solutions. We use this method to obtain policies generated by MEIR in the experiments.

\section{Generating anti-reward functions}
\subsection{Closed form solutions for $f$-divergence between two distributions.}

Different $f$-divergences and their closed-form solutions (for equation \ref{eq:f_div}) are present in Table \ref{table:closed_f_div}. This is adapted from the closed-form solutions of  $f$-divergences presented in \cite{iql}. 

\subsection{More on the use of IPMs}
In our experiments, we only use one type of IPM namely the Wasserstein-$1$ distance, which is obtained by constraining the family of functions $\mathcal{F}$ to be $1$-Lipschitz continuous. 

\label{app:KL}
\subsection{Idea behind Algorithm \ref{alg:gen_anti_reward}}
\label{app:idea_anti}
The concept behind algorithm \ref{alg:gen_anti_reward} becomes clearer through an example. Suppose we take the distribution $o$ to represent the occupancy measure. Initially, we can establish $\rho^-$ as the occupancy measure corresponding to a purely random policy $\hat{\pi}$. Consequently, the anti-reward function derived in Line 4 of Algorithm \ref{alg:gen_anti_reward} would attribute a low reward to state-action pairs visited by $\pi^*$, while granting higher rewards to other state-action pairs visited by $\hat{\pi}$. To exemplify this, if we employ KL divergence to quantify the divergence between $\rho^-$ and $\rho^*$, the anti-reward would adopt the structure of $\frac{\rho^-}{\rho^*}$. 

In the subsequent stage (Line \ref{div_line}), $\rho^{-}$ is redefined as the occupancy measure of the optimal policy of the newly computed anti-reward. In this context, state-action pairs assigned higher rewards are visited more frequently, while those receiving lower rewards (and thus frequented extensively by the optimal policy) are deliberately avoided.

\subsection{Example of an anti-reward}
\begin{center}
\begin{tikzpicture}[scale=0.2]
\tikzstyle{every node}+=[inner sep=0pt]
\draw [black] (21.7,-27.7) circle (3);
\draw (21.7,-27.7) node {$S_0$};
\draw [black] (38.1,-20.6) circle (3);
\draw (38.1,-20.6) node {$S_1$};
\draw [black] (38.1,-20.6) circle (2.4);
\draw [black] (38.1,-27.7) circle (3);
\draw (38.1,-27.7) node {$S_2$};
\draw [black] (38.1,-27.7) circle (2.4);
\draw [black] (38.1,-34.8) circle (3);
\draw (38.1,-34.8) node {$S_3$};
\draw [black] (38.1,-34.8) circle (2.4);
\draw [black] (23.122,-25.066) arc (144.82755:81.99062:12.631);
\fill [black] (35.21,-19.83) -- (34.48,-19.23) -- (34.34,-20.22);
\draw (27.22,-20.24) node [above] {$a_1$};
\draw [black] (24.7,-27.7) -- (35.1,-27.7);
\fill [black] (35.1,-27.7) -- (34.3,-27.2) -- (34.3,-28.2);
\draw (29.9,-28.2) node [below] {$a_2$};
\draw [black] (35.127,-35.167) arc (-88.63145:-138.18671:15.16);
\fill [black] (35.13,-35.17) -- (34.32,-34.69) -- (34.34,-35.69);
\draw (27.53,-34.44) node [below] {$a_3$};
\end{tikzpicture}

\end{center}

Consider the presented simple MDP, featuring an initial state $S_0$ and three terminal states $S_1, S_2, S_3$. In this context, if the original reward function is of the form $r = \left[1, 2, 3\right]$ for states $S_1, S_2$, and $S_3$ correspondingly, and the optimal policy $\pi^* = \left[0.01, 0.12, 0.87\right]$ emerges as the result of solving MERL with a minor weight attributed to the entropy term. The $\tau$-based anti-reward defined as $r^- = -\log(\pi^*(a|s))$, yields $r^- = \left[4.15, 2.15, 0.15\right]$.

\subsection{Comparison with MEIR}
\label{app:comp_meir}
Table \ref{table:comp_meir_app} presents a quantitative comparison between the MM and MEIR algorithms against an MCE-IRL based adversary with the true occupancy measure of the agent. Table \ref{table:comp_meir_app} highlights the quality of the recovered reward in terms of Pearson Correlation and EPIC distance \cite{epic} from the original reward, which are standard metrics for reward functions \cite{epic}. It can be inferred from Table \ref{table:comp_meir_app} that the MM algorithm consistently outperforms the MEIR algorithm across all the tested environments.

\subsection{Benchmarking Anti-Reward functions}
\label{app:bench}
The MM algorithm allows for the design of a wide-variety of anti-reward functions. Table \ref{table:comparing_occ} contains a comparison of the performance of a few of these anti-rewards against an MCE-IRL based observer with the true occupancy measures of the agent. Upon reviewing Table \ref{table:comparing_occ}, it becomes apparent that no singular anti-reward consistently results in the best performance. We believe this is because IRL algorithms fail to account for potentially misleading observations, rendering them poor observers. Consequently, we abstain from advocating for any single divergence as the definitive choice.

\section{Action predictability}
\label{boltzman_exp}
Unlike the MEIR algorithm that results in \textit{stochastic} policies, the MM algorithm results in \textit{deterministic} policies as the algorithm optimizes a scalar reward function (Equation \ref{eq:mm_dual}). This makes action predictability easier for the observer, which could pose a security risk.

We present two methods to deal with this: The first is based on the agent using a mixed policy of multiple solutions of MM and the second is based on sampling of actions according to the Boltzmann distribution \cite{be}.

\subsection{Mixed Policy method}
\begin{figure*}
\includegraphics[width=2\columnwidth]{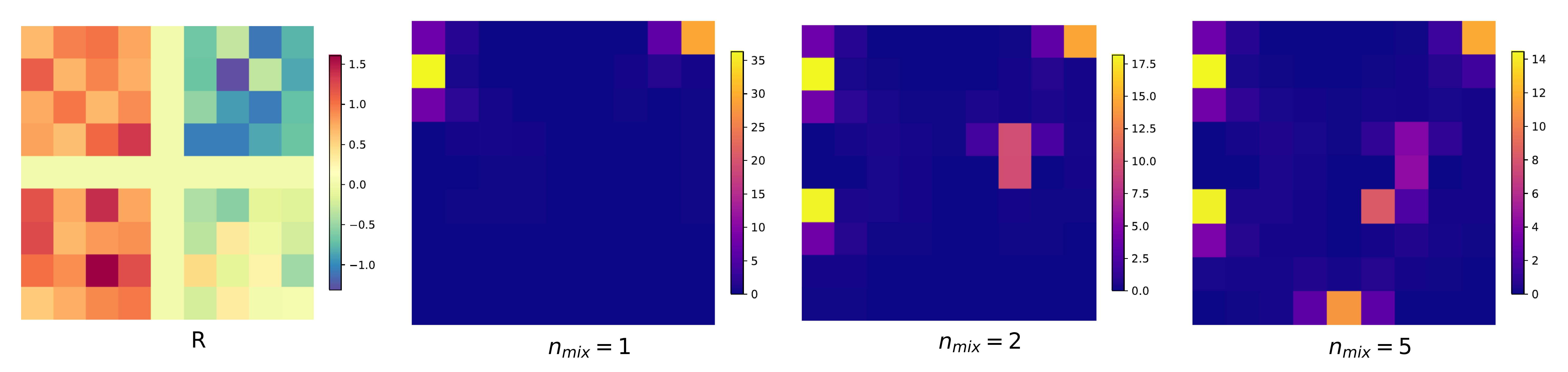}
\caption{Mixing multiple solutions of the MM algorithm results in the agent visiting a more diverse set of misleading states.}
\label{fig:comp_modes}
\end{figure*}

It is important to note that there can be multiple solutions when generating an anti-reward using Algorithm \ref{alg:gen_anti_reward}. This is because, there are multiple states/trajectories that are \textit{far away} from the optimal states/trajectories, and each solution is an anti-reward that prefers just one of the regions, as can be seen in Figure \ref{fig:comp_modes}. To overcome this, we can generate multiple anti-rewards $r^-_i$ (by setting a different seed when initializing the neural network) and for each of them generate a deceptive policy $\bar{\pi}_i$ with the same $E_{min}$ value. The agent can then use a mixed policy $\bar{\pi}^{mix} = \{\{\bar{\pi_i}\}, w\}$ where $w$ can be a uniform distribution or manually chosen by the user based on their preference of each policy. We call this method $\text{MM}^{mix}$. Note that $\bar{\pi}^{mix}$ satisfies a reward threshold of $E_{min}$ as each of its constituent policies $\bar{\pi_i}$ satisfy the reward constraint.

\begin{proposition}
If a set of policies $\{\pi_i\}$ satisfy a constraint on the expected reward, i.e, $E_{\pi_{i}}[r] \geq E_{min}$ $\forall i$, then the mixed policy, $\pi^{mix} = \{\{\pi_i\}, w\}$ satisfies the constraint on expected reward, i.e, $E_{\pi_{mix}}[r] \geq E_{min}$.
\end{proposition}
\begin{proof}
Let $x_{i}$ be the occupancy measure of $\pi_i$, then, the occupancy measure of the mixed policy $x_{mix}(s,a) = \sum_i w_i x_i(s,a)$ \cite{10.1145/1390156.1390286}, 
\begin{align*}
    \sum_s \sum_a x_i (s,a) r(s,a)  &\geq E_{min} \quad \forall i \\
    \sum_i w_i \sum_s \sum_a x_i (s,a) r(s,a) &\geq \sum_i w_i E_{min} \\
    \sum_s \sum_a (\sum_i w_i x_i (s,a)) r(s,a) &\geq (\sum_i w_i) E_{min} \\
    \sum_s \sum_a x_{mix} (s,a) r(s,a) &\geq E_{min}
\end{align*}
\end{proof}

\subsection{Boltzmann Randomization}

We can make the policy more stochastic by sampling actions according to the Boltzmann distribution \cite{be} during inference. The Boltzmann distribution is defined as, 
\begin{equation}
    BE(r, \beta) = \bar{\pi}_{BE}(a|s) = \frac{e^{\frac{1}{\beta} Q^* (s,a)} } {\sum_{a' \in A} e^{\frac{1}{\beta} Q^* (s,a')}}
\end{equation}
where $Q^*$ is the optimal $Q$-function and $\beta$ is a temperature parameter used to control the stochasticity of the policy distribution. Based on this, we introduce the Max Misinformation with Boltzmann Exploration algorithm as, 
\begin{equation}
    MMBE(r, E_{min}, \beta) = BE(\bar{r}, \beta) \circ MM(r, E_{min}) 
\end{equation}
where, $\bar{r} = \lambda^*r + r^-$. That is, we first run the MM algorithm and then employ BE on top of the Q function (obtained using MM algorithm) to obtain a randomized policy. Most value methods, learn a Q-function when solving the objective in Lines \ref{conv_line} in \ref{conv_line2} in Algorithm \ref{alg:mm_pd_descent}. In case a Q-function is not available, it can be estimated using Monte Carlo rollouts. However, randomizing the policy in this way does not guarantee that the reward constraint is respected leading to a tradeoff between increasing the stochasticity of the policy and obeying the reward constraint. 

\subsubsection{Experiments on reducing action predictability}

\begin{figure*}
\centering
\includegraphics[width=0.7\textwidth]{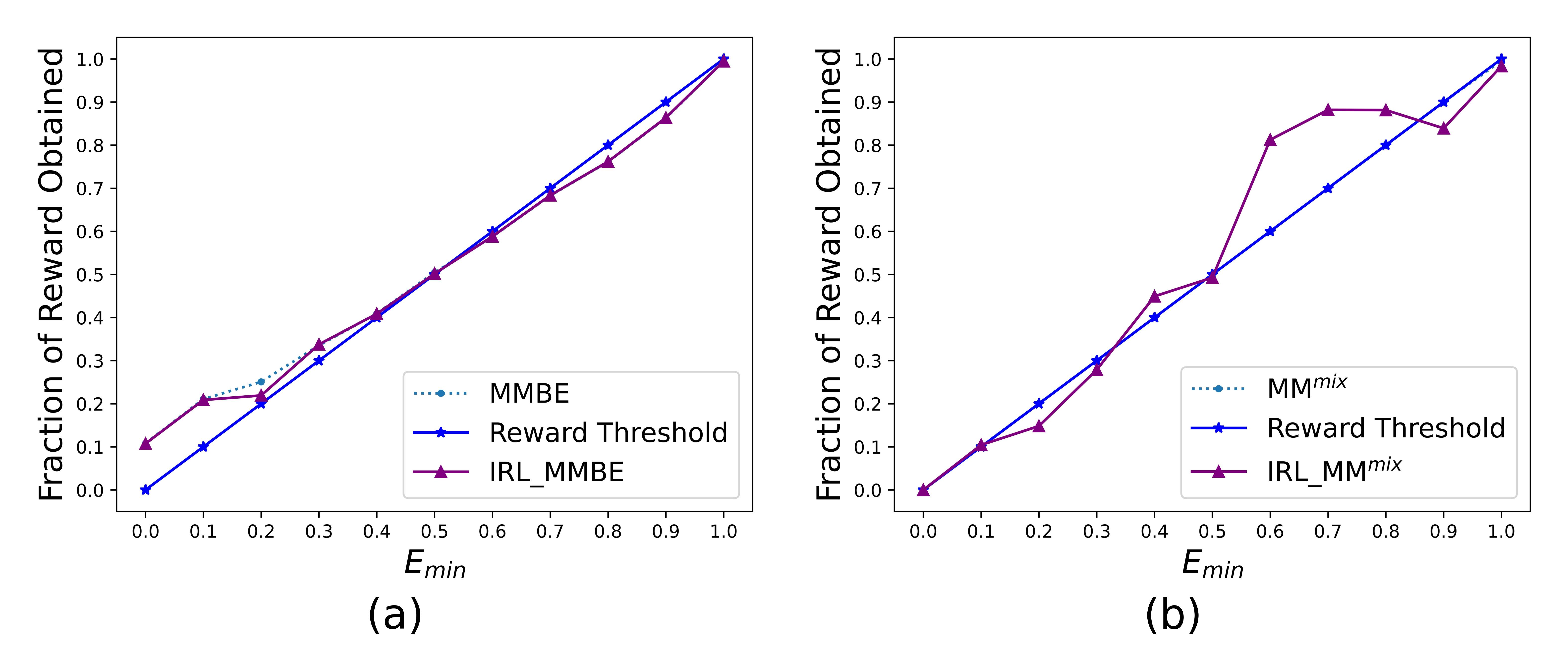}
\caption{Comparison of (a) MMBE with $\beta=0.1$ and (b) Mixed policy containing $5$ MM policies.}
\label{fig:mmbe}
\end{figure*}

\begin{figure}
    \centering
    \includegraphics[scale=0.5]{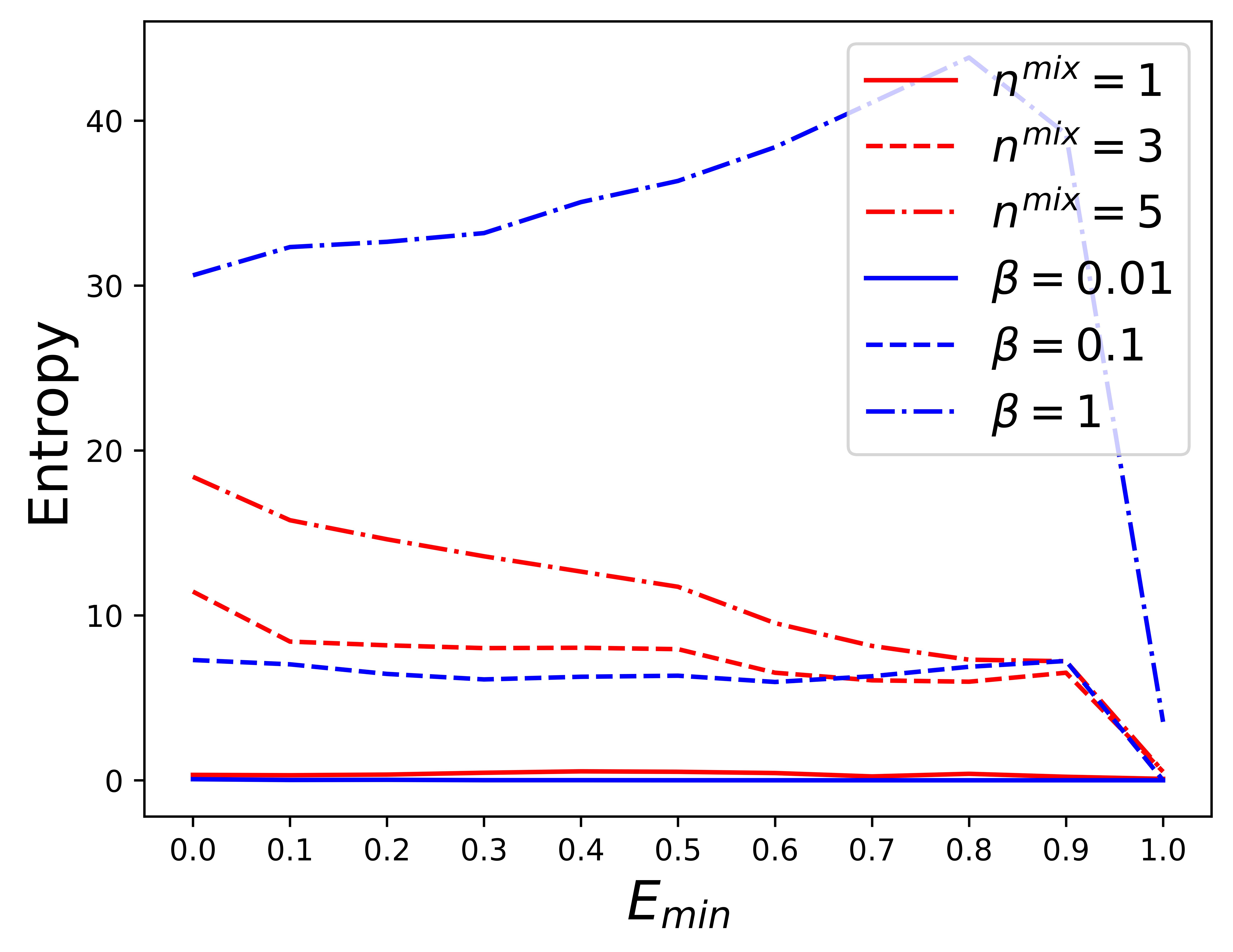}
    \caption{Comparison of the entropy of policies with varying $\beta$ and $n^{mix}$ values generated by MMBE and $\text{MM}^{mix}$ in the Four Rooms environment.}
    \label{fig:comp_ent}
\end{figure}

Figure \ref{fig:mmbe} illustrates that both the MMBE and $\text{MM}^{mix}$ algorithms achieve privacy performance comparable to the MM algorithm. However, the MMBE algorithm exhibits a slight deviation from the reward constraint, as indicated in the previous section.

In scenarios where satisfying the reward constraint is crucial, the $\text{MM}^{mix}$ algorithm proves more suitable as it is guaranteed to meet the reward constraint. However, it is worth noting that the computational complexity of $\text{MM}^{mix}$ is higher, requiring the execution of the $MM$ algorithm $n^{mix}$ times, in contrast to a single run in the case of the MMBE algorithm. Therefore, MMBE is preferable in scenarios with limited computational resources and where minor deviations from the reward constraint are acceptable.

\section{Environments}
\label{app:envs}

We use the following environments in our experiments. 

\subsection{Cybersecurity}

The Cybersecurity environment is based on the Moving Target Defense problem \cite{mtd} where a network administrator dynamically switches between network configurations to make it difficult for an attacker to infiltrate a network and compromise its nodes. We consider Active Directory networks \cite{Active_dir} created using the CyberBattle \cite{msft:cyberbattlesim} library, which generates these networks based on models of real-world networks. 

Each network configuration $n$ is defined by a set of vulnerabilities and active connections between nodes. Let $\mathcal{N} = \{n\}$ be the set of all network configurations available to the agent. Each configuration possesses an intrinsic security value $v$, indicative of the ease with which an attacker can compromise nodes post-infiltration. To compute these values, we leverage an attacker from the CyberBattle library, developed through transfer learning on numerous active directory networks. This attacker exhibits the ability to generalize attacks across diverse network configurations and sizes. Consequently, for each configuration $n \in \mathcal{N}$, we emulate an attack scenario using the aforementioned agent and assign a value inversely proportional to the compromised node count resulting from the attack. Additionally, a minor weight is assigned to the count of active connections. This ensures that the agent avoids selecting configurations without any active connections, which hold no practical value.

The state space of the MDP is defined by the tuple $(n, p)$, where $n$ is a network configuration and $p$ is its protection level. The protection level $p$ is used to account for the uncertainty in the security value of a network configuration in real-time since an arbitrary number/type of attackers can try to compromise the network that have not been accounted for. In each state, the agent can take an action to switch to any other configuration but will have an unknown protection level. Hence, the size of the action space $|A| = |S|$. Consider a state $(n_i, p_1)$, let $a_{n_j}$ indicate an action to switch to network configuration $n_j$. Then the transition matrix $P((n_i, p_1), a_{n_j}, (n_l, p_2)) = 0$ if $l \neq j$ and $\geq 0$ otherwise. In our experiments, we consider, two levels of protection, i.e, $|p| = 2$. The reward for each state $r\left((n, p)\right) = v$ if $p=1$ and $r\left((n, p)\right) = v - \delta$ if $p=0$. Each episode begins in an arbitrary configuration.

The CyberBattle library offers 10 preset Active Directory networks. For each network, we generate 25 different configurations. 

\subsection{Four Rooms}
The Four Rooms environment is a part of the Minigrid \cite{minigrid} set of environments. This environment consists of 4 rooms separated by walls as shown in Figure \ref{fig:envs}. Each wall has a door that is randomly placed on the wall separating the two rooms that allows the agent to traverse between the rooms. Each room $i$ is given a security value $S_{i}$ and a variance value $\sigma_{i}$, and each cell $j$ in room $i$ is given a value $v_{ij} \sim \mathcal{N}(S_{i}, \sigma_{i})$. Hence, the value of a cell depends on the security value assigned to its room. The agent starts the episode in any of the four rooms with equal probability and receives a reward $r(s_{ij}) = v_{ij}$ after visiting cell $(i,j)$. 

\subsection{Frozen Lake}
The Frozen Lake environment is part of the Gymnasium Library \cite{gym}. This environment is represented as a 2D grid as shown in Figure \ref{fig:envs}. The ice is slippery which introduces transition uncertainty and filled with holes. Each hole is assigned a reward of $-1$, and is absorbing, i.e, the agent is stuck in the hole once it enters and receives a reward of $-1$ every time step till the end of the episode. Every other cell is assigned a reward randomly sampled from the range $[0,1]$ which the agent receives after visiting that state. The agent starts each episode at one of the cells that do not contain a hole.

The utilization of the Four Rooms and Frozen Lake environments is rooted in addressing real-world challenges such as police patrolling (Chen, 2013) and Green Security Games (GSG) scenarios \cite{GSG_1}. These scenarios involve discretizing the patrol region into 2D grid cells, where each cell is assigned a value reflective of factors like crime or animal density—information we aim to keep confidential. Also, the Four Rooms environment includes physical barriers that mimic real-world obstacles. The Frozen Lake environment poses an additional challenge, as agents are restricted from taking highly sub-optimal paths to deceive observers; they must avoid landing in holes to maintain the reward constraint.

\subsection{Random MDP}
These environments, introduced in \cite{maxent}, were used to evaluate the MEIR algorithm. We generate random MDPs with $|S|$ ranging from $28 $to $40$ and $|A|$ ranging from $2$ to $15$. Each state is assigned a reward $r(s) \in [-1, 1]$, and episodes always start at state $0$.

\begin{figure}
\centering
\includegraphics[width=\columnwidth]{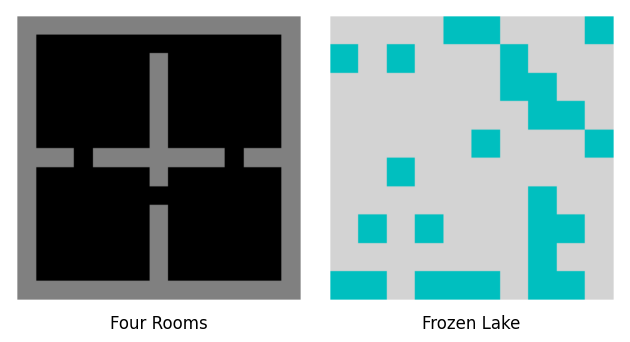}
\caption{Evaluation environments}
\label{fig:envs}
\end{figure}

The results are averaged over 5 seeds of the Random MDP, Frozen Lake and Four Rooms environments and over 10 networks in the case of the Cyber Security environment. In the comparison of MEIR and MM with an MCE IRL-based observer, results are presented for grid sizes of $10\times10$ and $13\times13$ for the Frozen Lake and Four Rooms environments respectively. In the case of an IQ-Learn-based observer and comparison with the DQFN algorithm, results are presented for grid sizes $5\times5$ and $9\times9$ for the Frozen Lake and Four Rooms environments due to the large run times of IQ-Learn and DQFN. 

Tabular Q-learning was used to solve the RL objective (Lines \ref{conv_line} and \ref{conv_line2} in Algorithm \ref{alg:mm_pd_descent} and \ref{alg:mm_binary_search}). The $f$-divergence based anti-rewards were derived using their closed form solutions (Table \ref{table:closed_f_div}), and the Imitation library \cite{gleave2022imitation} implementation of MCE IRL was used. The DQFN and IQ-Learn algorithms were implemented using the Stable-Baselines3\cite{stable-baselines3} library. The experiments were run an a machine with $8$ CPUs and $2$ NVIDIA GeForce GTX 1080 GPUs.

\begin{table*}[t]
\centering
\caption{Different $f$-divergences and their corresponding closed form solutions.}.
\label{table:closed_f_div}
\begin{tabular}{|l|c|c|c|}
\hline
Divergence & \textbf{$f(t)$} & \textbf{$\boldsymbol{\phi}(u)$} & \textbf{$r^-$}\\
\hline
Forward KL & $-\log{t}$ & $1 + \log{u}$ & $\frac{\rho^-}{\rho^*}$ \\[2mm]
Backward KL & $t\log{t}$ & $-e^{-{u+1}}$ &  $-(1+\log{\frac{\rho^*}{\rho^-}})$ \\[2mm]
Jensen Shannon & $-(t+1)\log(\frac{1+t}{2}) + t\log{t}$ & $\log(2-e^{-u})$ &  $\log\frac{1}{2}(1+{\frac{\rho^-}{\rho^*}})$ \\[2mm]
Pearson $\chi^2$ & $(t-1)^2$ & $u-\frac{u^2}{4}$ &  $2(1-\frac{\rho^*}{\rho^-})$ \\[2mm]
Squared Hellinger & $(\sqrt{t}-1)^2$ & $\frac{u}{1+u}$ &  $\sqrt{\frac{\rho^-}{\rho^*}}-1$ \\[2mm]
Total Variation & $\frac{1}{2}|t-1|$ & $u$ &  $1-\text{sign}(1-\frac{\rho^*}{\rho^-})$ \\[2mm]
\hline
\end{tabular}
\end{table*}

\begin{table*}[t]
\centering
\caption{Comparison with the MEIR algorithm}
\label{table:comp_meir_app}
\begin{tabular}{|llcc|}
\hline
\textbf{Env} & \textbf{Algorithm} & \textbf{Avg. Pearson} & \textbf{Avg. EPIC} \\
\hline
\multirow{4}{*}{Random MDP} & MEIR & 0.98 & 0.04 \\
 & MM ($\tau$-based) (KL) & 0.58 & 0.45 \\
 & MM ($\rho$-based) (W-1) & 0.77 & 0.33 \\
 & MM ($\rho$-based) (JS) & \textbf{0.43} & \textbf{0.51} \\
\hline
\multirow{4}{*}{Frozen Lake $5\times5$} & MEIR & 0.90 & 0.15 \\
 & MM ($\tau$-based) (KL) & \textbf{0.17} & \textbf{0.64} \\
 & MM ($\rho$-based) (W-1) & 0.46 & 0.52 \\
 & MM ($\rho$-based) (JS) & 0.25 & 0.60 \\
\hline
\multirow{4}{*}{Frozen Lake $10\times10$} & MEIR & 0.56 & 0.46 \\
 & MM ($\tau$-based) (KL) & \textbf{0.14} & \textbf{0.65} \\
 & MM ($\rho$-based) (W-1) & 0.26 & 0.61 \\
 & MM ($\rho$-based) (JS) & 0.18 & 0.64 \\
\hline
\multirow{4}{*}{Four Rooms $9\times9$} & MEIR & 0.77 & 0.29 \\
 & MM ($\tau$-based) (KL) & \textbf{0.05} & \textbf{0.69} \\
 & MM ($\rho$-based) (W-1) & 0.16 & 0.64 \\
 & MM ($\rho$-based) (JS) & 0.18 & 0.64 \\
 \hline
\multirow{4}{*}{Four Rooms $13\times13$} & MEIR & 0.41 & 0.52 \\
 & MM ($\tau$-based) (KL) & \textbf{0.05} & \textbf{0.69} \\
 & MM ($\rho$-based) (W-1) & 0.11 & 0.67 \\
 & MM ($\rho$-based) (JS) & 0.11 & 0.67 \\
\hline
\multirow{3}{*}{Cyber Security} & MEIR & \textbf{0.09} & \textbf{0.70} \\
 & MM ($\tau$-based) (KL) & 0.22 & 0.67 \\
 & MM ($\rho$-based) (W-1) & 0.46 & 0.51 \\
\hline
\end{tabular}
\end{table*}

\begin{table*}[t]
\centering
\caption{Comparing anti-reward functions based on occupancy measures}
\label{table:comparing_occ}
\begin{tabular}{|llcc|}
\hline
\textbf{Env} & \textbf{Divergence} & \textbf{Avg. Pearson} & \textbf{Avg. EPIC} \\
\hline
\multirow{8}{*}{Random MDP} & Wasserstein-1 (Linear) & 0.77 & 0.33 \\
 & Wasserstein-1 (NN) & 0.75 & 0.34 \\
 & Forward KL & 0.64 & 0.40 \\
 & Backward KL & 0.50 & 0.48 \\
 & Hellinger & 0.58 & 0.44 \\
 & Pearson $\chi^2$ & 0.70 & 0.37 \\
 & Total Variation & 0.66 & 0.40 \\
 & Jensen Shannon & \textbf{0.43} & \textbf{0.51} \\
\hline
\multirow{8}{*}{Four Rooms $9\times9$} & Wasserstein-1 (Linear) & 0.17 & 0.64 \\
 & Wasserstein-1 (NN) & 0.26 & 0.60 \\
 & Forward KL & 0.26 & 0.60 \\
 & Backward KL & 0.19 & 0.64 \\
 & Hellinger & 0.23 & 0.62 \\
 & Pearson $\chi^2$ & 0.18 & 0.63 \\
 & Total Variation & \textbf{0.07} & \textbf{0.68} \\
 & Jensen Shannon & 0.18 & 0.64 \\
\hline
\multirow{8}{*}{Frozen Lake $5\times5$} & Wasserstein-1 (Linear) & 0.46 & 0.52 \\
 & Wasserstein-1 (NN) & 0.46 & 0.52 \\
 & Forward KL & 0.30 & 0.58 \\
 & Backward KL & 0.23 & 0.61 \\
 & Hellinger & 0.26 & 0.60 \\
 & Pearson $\chi^2$ & \textbf{0.17} & \textbf{0.63} \\
 & Total Variation & 0.20 & 0.62 \\
 & Jensen Shannon & 0.25 & 0.60 \\ 
 \hline
\end{tabular}
\end{table*}

\section{Related Work}
\label{app:related_work}

There exist two primary approaches towards reward function privacy: Differential Privacy (DP) and Deception.

\subsection{Differential Privacy based methods}

Differential Privacy \cite{10.1007/11681878_14} is a mathematically rigorous framework that ensures that the output of a "mechanism" when applied to a dataset is robust to changes in individual data points of the dataset. This property makes it difficult for observers to reconstruct individual data points from the output, thereby preserving privacy. The application of differential privacy in RL has been studied in prior works \cite{Prakash_Husain_Paruchuri_Gujar_2022, vietri2020private, Zou, dqfn}. The approach involves introducing Gaussian noise to Bellman updates \cite{Prakash_Husain_Paruchuri_Gujar_2022, dqfn}, gradients of TD-errors and surrogate advantages \cite{Prakash_Husain_Paruchuri_Gujar_2022} to ensure differential privacy. By doing so, the Value functions, and post processed functions such as policies/occupancy measures become resistant to observers attempting to recover the exact reward function by querying these functions. In the context of reward function privacy, DP based security guarantees are ill-suited as: (a) there are infinitely many reward functions that can represent the user's preferences (e.g., shaped rewards), so preventing the observer from reconstructing the exact reward function does not fully address the issue and (b) the $l_{\infty}$ and $l_{p}$ norms are not good metrics to use when comparing reward functions \cite{epic} as two reward functions in the same $l_{\infty}$ neighbourhood may possess several other properties that pose a privacy leak such as ordering of polices, hence more nuanced metrics such as those presented in Section \ref{sec:assesing} are required. Apart from this, it has been demonstrated in \cite{Prakash_Husain_Paruchuri_Gujar_2022} that policies learnt using DP-based RL algorithms possess unfavourable properties from the standpoint of usability and privacy: (a) the expected reward does not monotonically increase with the increase in the privacy budget, making it hard for the user to control the privacy/reward tradeoff, (b) the quality of the reward function recovered by an observer is independent of the privacy budget, thus "rendering all strategies ineffective at being a truly meaningful private strategy"-\cite{Prakash_Husain_Paruchuri_Gujar_2022}. We show empirically that the MM algorithm does not suffer from these drawbacks. Also, methods such as \cite{Prakash_Husain_Paruchuri_Gujar_2022, vietri2020private, Zou, dqfn} are specific to the learning algorithm, i.e, value-based or policy gradient-based, whereas, the MM algorithm can be implemented with any learning algorithm. 


\subsection{Deception based methods}

Deception can be categorized as the act of creating or maintaining false beliefs in the minds of others \cite{10.1093/acprof:oso/9780199577415.001.0001}. Military strategists Bell and Whaley \cite{article, bowyer1982cheating, doi:10.1080/01402398208437106} argue that the "distortion of perceived reality" can be achieved through two methods: \textit{simulation} (showing the false) and \textit{dissimulation} (hiding the truth). Deceptive Reinforcement Learning introduced in \cite{deceptiverl} is a paradigm within RL that uses the concept of deception to learn policies that preserve the privacy of the reward function. Previous research has explored model-based deception-based path planning \cite{10.5555/3061053.3061065, kulkarni2018resource, Masters2017DeceptiveP, topcu}, whereas the algorithm presented in this research is model-free. \cite{deceptiverl} present two methods based on \textit{simulation} and \textit{dissimulation} for model-free path planning. Our work differs from these path-planning algorithms as follows:

(a) Path planning algorithms consider environments with terminal goal states and aim to maximize the uncertainty of the goal's location within a single trajectory. We assume that the observer is observing the agent over multiple trajectories, in such a case, every demonstration would end at a goal state and identifying the location of the goal states becomes trivial.

(b) The difference in the underlying assumptions and constraints between the work presented in \cite{deceptiverl} and our work (e.g., agents behaviour governed by the reward functions within a predefined set) limit the direct applicability of the findings from \cite{deceptiverl} to the context of IRL-based observers. Our research explores and addresses the unique challenges IRL algorithms pose when presented with multiple demonstrations from a private policy.

(c) The algorithms presented in \cite{deceptiverl}, does not incorporate constraints on the expected reward. Hence, similar to DP-based algorithms, one must carefully tune the hyperparameters to balance the privacy-reward trade-off.

The MEIR algorithm, presented in \cite{maxent}, is another existing approach towards policy optimization for preserving privacy against unknown observers. They work under the same assumption i) observer is unknown and (ii) observer has access to the agent’s policy and the goal is to minimise the information provided to the observer to enhance security. MEIR focuses on maximizing the entropy of the agent's policy for privacy preservation. We demonstrate that the MEIR algorithm can be reformulated as a dissimulation-based Deceptive RL approach that incorporates reward constraints. This formulation is applicable in the IRL-based observer setting, as it aims to preserve privacy across multiple demonstrations. Section \ref{sec:meir} of this paper proves that the policies learnt using the MEIR algorithm exhibit a significant privacy leak of the reward function when confronted with observers employing the MCE IRL algorithm.

\end{document}